\documentclass[twoside,11pt]{article}

\usepackage[abbrvbib, preprint]{jmlr2e}
\setcitestyle{square,comma, open={[}, close={]}}
\setcitestyle{numbers}
\usepackage{hyperref}
\usepackage{tabularx}
\usepackage{amsmath, amssymb}
\usepackage{hyperref}
\hypersetup{hidelinks}
\usepackage[ruled,vlined]{algorithm2e}
\usepackage{mathrsfs}
\usepackage{pgfplots}
\usepackage[skip=2pt,font=scriptsize]{caption}
\usepackage{float}
\usepackage{bbm}
\usepackage{bm}
\usepackage{dsfont}

\usepackage{lipsum}
\usepackage [english]{babel}
\usepackage [autostyle, english = american]{csquotes}

\newtheorem{thm}{Theorem}[section]{\bfseries}{\itshape}
\newtheorem{lem}{Lemma}[section]{\bfseries}{\itshape}
\newtheorem{prop}{Proposition}[section]{\bfseries}{\itshape}
\newtheorem{coro}{Corollary}[section]{\bfseries}{\itshape}
\newtheorem{defi}{Definition}[section]{\bfseries}{\itshape}
\newtheorem{rema}{Remark}[section]{\bfseries}{\itshape}
\newtheorem{exam}{Example}[section]{\bfseries}{\itshape}

\usepackage[noend]{algpseudocode}
\algnewcommand\algorithmicinput{\textbf{Input:}}
\algnewcommand\algorithmicoutput{\textbf{Output:}}
\algnewcommand\Input{\item[\algorithmicinput]}%
\algnewcommand\Output{\item[\algorithmicoutput]}

\DeclareMathOperator{\var}{Var}
\DeclareMathOperator{\Var}{Var}

\DeclareMathOperator{\Id}{Id}

\DeclareMathOperator*{\argmin}{arg\,min}
\DeclareMathOperator{\bary}{Bary}

\newcommand{\edit}[1]{\textcolor{blue}{[#1]}}

\newcolumntype{b}{X}
\newcolumntype{s}{>{\hsize=.5\hsize}X}
\newcolumntype{t}{>{\hsize=.3\hsize}X}

\newcommand{\ts}[1]{\textcolor{red}{[TS: #1]}}

\jmlrheading{}{}{}{}{}{}{Shizhou Xu}

\ShortHeadings{Optimal Data Partitioning via Wasserstein Homogeneity}{Xu and Strohmer}
\firstpageno{1}

\begin{document}

\title{WHOMP: Optimal Data Partitioning via Wasserstein Homogeneity}
\title{WHOMP: Optimizing Randomized Controlled Trials via Wasserstein Homogeneity}

\author{\name Shizhou Xu \email shzxu@ucdavis.edu \\
       \addr Department of Mathematics\\
       University of California Davis\\
       Davis, CA 95616-5270, USA
       \AND
       \name Thomas Strohmer \email strohmer@math.ucdavis.edu \\
       \addr Department of Mathematics\\
       Center of Data Science and Artificial Intelligence Research\\
       University of California Davis\\
       Davis, CA 95616-5270, USA}

\maketitle

\begin{abstract}
We investigate methods for partitioning datasets into subgroups that maximize diversity within each subgroup while minimizing dissimilarity across subgroups. We introduce a novel partitioning method called the \textit{Wasserstein Homogeneity Partition} (WHOMP), which optimally minimizes type I and type II errors that often result from imbalanced group splitting or partitioning, commonly referred to as accidental bias, in comparative and controlled trials. We conduct an analytical comparison of WHOMP against existing partitioning methods, such as random subsampling, covariate-adaptive randomization, rerandomization, and anti-clustering, demonstrating its advantages. Moreover, we characterize the optimal solutions to the WHOMP problem and reveal an inherent trade-off between the stability of subgroup means and variances among these solutions. Based on our theoretical insights, we design algorithms that not only obtain these optimal solutions but also equip practitioners with tools to select the desired trade-off. Finally, we validate the effectiveness of WHOMP through numerical experiments, highlighting its superiority over traditional methods.
\end{abstract}


\vfill
\newpage

\tableofcontents
\newpage

\section{Introduction}

Congratulations! After investing years of hard work and hundreds of millions of dollars, your company has discovered a promising new cancer drug. The next milestone is to conduct a {\em randomized clinical trial} to confirm the drug's effectiveness.  
However, occasionally the randomization procedure can cause an imbalance in covariates related to the outcome across groups. 
A chance you are reluctant to take, since too much is at stake here! You are, of course, aware of the various attempts to mitigate the potential downsides of randomization, such as covariate adaptive randomization. But these alternatives have their own drawbacks, often seem ad hoc, and very rare of these methods are designed with any optimality criteria for comparative tests. Enter WHOMP,  {\em Wasserstein HOMogeneous Partitioning}, a method that constructs maximally balanced data partitions with provable optimality guarantees.

\subsection{Motivation}

Randomized group splitting has been a widely accepted standard for estimating causal inference in scientific experiments, as randomization typically balances covariate effects between group divisions and experimental outcomes on average over repeated trials. However, the risks associated with pure randomization and imbalanced group splitting have been highlighted in numerous studies across fields such as agriculture, biology, social sciences, and clinical research \cite{fisher1992arrangement, krause2003random, yates1939comparative, 10.1214/08-STS269}. 

The widely held belief behind randomization is that it promotes comparability between the resulting subgroups. For instance, Rosenberger states in \cite{rosenberger2015randomization}, “The first property of randomization is that it promotes comparability among the study groups.” However, this result holds with reasonably high probability only when the law of large numbers applies to the randomized subsampling process. In many controlled trials, the sample size is inherently limited. Additionally, conducting repeated experiments with randomized sample splitting can be prohibitively expensive or even impractical in many scientific settings. Therefore, the law of large numbers may not apply to either the group size or the number of trials. As stated by Fisher, who first proposed the requirement of randomization in experimental design, in \cite{fisher1992arrangement}, ``Most experimenters on carrying out a random assignment of plots will be shocked to find how far from equally the plots distribute themselves''. When the (sub)sample size is insufficient or the number of covariates is relatively large, there is a non-negligible chance that the randomization itself becomes a covariate factor in a limited number of realizations, potentially leading to type~I or type~II errors.

This work aims to address the following question, which naturally arises in scientific experiments and causal inference studies:
\begin{itemize}
\item[] \textit{How can we split a sample into control and test (or multiple controlled) groups in a way that minimizes the impact of the data splitting on the outcomes of the controlled experiment?}
\end{itemize}
We approach this question from two key perspectives:
\begin{itemize}
\item \textit{In-subgroup diversity}: Maximizing diversity within each subgroup (or partition element) based on a specific diversity metric.
\item \textit{Cross-subgroup similarity}: Minimizing dissimilarity across subgroups using a defined similarity measure.
\end{itemize}
Here, maximizing diversity within each subgroup ensures that the test results are more representative of the entire sample, which is often assumed to reflect the target population. At the same time, minimizing dissimilarity between subgroups, where different controlled factors are applied, reduces the likelihood that the statistical (in)significance is driven by covariate imbalances introduced through group splitting.

Beyond scientific experiments, the study of group splitting to maximize in-subgroup diversity and cross-subgroup similarity has garnered attention in various fields: \textit{Graph Theory:} Partitioning the nodes of a (weighted) graph into clusters such that the total weight of edges with both endpoints in the same cluster is maximized \cite{feo1990class, feo1992one}. \textit{Federated Learning:} Identifying “superclients” to address distribution heterogeneity in training data across clients, using either unsupervised approaches \cite{mohri2019agnostic, zaccone2022speeding} or supervised methods \cite{du2021fairness}. \textit{Managerial Science:} Promoting diversity within workgroups to enhance productivity \cite{baker2002methods, bhadury2000maximizing, fan2011erratum}.

In this work, we propose a new partitioning objective that addresses both perspectives:
\begin{itemize}
\item[] \textit{Homogeneity Partition}: Given a distance metric on probability distribution spaces, such as Wasserstein spaces, the partitioning method aims to minimize the average squared distance between the entire sample and the resulting subgroups.
\end{itemize}
Here, in-subgroup diversity is captured by minimizing the distance between the subgroup and the entire sample: less diversity (relative to the entire sample) in a subgroup results in a greater distance between it and the entire sample. On the other hand, cross-subgroup similarity is captured by the minimization of average squared distance: The average squared distance minimizes the variability among the subgroups around the sample. The distance is squared because $\ell^2$-minimization promotes a more uniform distribution of distance or variability, compared to $\ell^1$-minimization, and a more balanced distribution of the distributional metric results in cross-subgroup similarity.

In this study, we concentrate on the Wasserstein-2 distance and present a comprehensive analysis of the above considerations. The main contributions of our work are as follows:
\begin{itemize}
    \item \textbf{Optimality criterion} (Section \ref{S:2}): We introduce a novel partitioning objective, named WHOMP, which is designed to establish a provable optimality criterion that guarantees the effectiveness of comparative experiments through the resulting partitions.
    \item \textbf{Analytical comparison} (Section \ref{S:3}): We provide a thorough analytical comparison of widely used partitioning methods in comparative experiments, including random partitioning, covariate-adaptive randomization, rerandomization, and anti-clustering, with the proposed WHOMP. This analysis highlights the connections, distinctions, and advantages of WHOMP in comparison to existing methods.
    \item \textbf{Solution characterization and algorithm design} (Sections \ref{S:4}, \ref{S:5}): We characterize the optimal solutions to the WHOMP problem and develop an efficient algorithm for their estimation. Furthermore, we identify a trade-off among different WHOMP optimal solutions, offering guidance for practitioners in selecting the most appropriate solution for specific use cases.
    \item \textbf{Numerical Comparison} (Section \ref{S:6}): We perform a numerical comparison of the standard partitioning methods and the WHOMP (implemented with our algorithm design) across various data types, including tabular, image, and graph data.
\end{itemize}

\subsection{Related Work}

One line of research related to diversified subgroup generation involves balancing significant covariates during randomized group splitting. This idea goes back at least to \cite{efron1971forcing} and has been widely employed in various comparative studies, including clinical trials \cite{rosenberger2015randomization}, A/B testing for business decisions \cite{wang2023b}, and experiments in the social sciences \cite{duflo2007using}. The objective is to balance covariates that may influence the results during randomized group splitting, thereby enhancing the credibility and efficiency of the trial or experiment. In other words, the goal is to reduce type I and type II errors caused by covariate imbalances in group assignments. To achieve this, methods such as covariate-adaptive randomization \cite{ma2024new}, block-stratified randomization \cite{kernan1999stratified}, and minimization \cite{scott2002method, coart2023minimization} are commonly employed. Despite their extensive application, these methods have faced criticism for lacking optimality criteria related to guaranteed comparative test performance \cite{senn2013seven, 10.1214/08-STS269, 10.1214/12-AOS983}. Existing methods either reduce distributional similarity to similarity in the first moments, as in minimization \cite{taves1974minimization, pocock1975sequential}, or rely on the assumption of a specific model for treatment effects \cite{atkinson1982optimum}.

Another line of research related to this work focuses on maximizing in-subgroup diversity, though these problems are studied under various terms, such as \textit{anti-clustering} \cite{spath1986anticlustering, papenberg2024k}, \textit{K-partition} \cite{feo1992one}, \textit{equitable partition} \cite{o1997heuristic}, and \textit{maximally diverse group problem} \cite{brimberg2015solving, gallego2013tabu, rodriguez2013artificial}. The distinction among these problems is that some consider a more general distance or diversity penalty function beyond the Euclidean distance or variance. It is important to note that when enforcing uniform cardinality of subgroups, all these problems are equivalent in the Euclidean setting. Therefore, we use the term anti-clustering to represent this body of work and explore the similarities and differences between WHOMP and anti-clustering to highlight the advantages of the proposed method.

In particular, we highlight a common misunderstanding in the current anti-clustering approach to the diverse subgroups problem. As discussed earlier, the problem of subgroup splitting for comparative tests should encompass two aspects that are not necessarily compatible: \textit{In-subgroup diversity} and \textit{cross-subgroup similarity}. The anti-clustering approach primarily focuses on in-subgroup diversity but relies on the following duality result to argue that maximizing in-subgroup diversity also maximizes cross-subgroup similarity: Maximizing in-subgroup variance is equivalent to minimizing the variance of the centroids across subgroups. For the exact statement, see Lemma \ref{centroid variable characterization} or \cite{spath1986anticlustering}. That is, this equivalence holds only when cross-subgroup similarity is defined as similarity in subgroup averages.

However, enforcing similarity among subgroup averages often leads to scale differences among subgroups. Points with similar scales but different directions tend to be grouped together to balance each other out, thereby achieving similarity in averages. While this scale matching can be beneficial for certain applications, the resulting cross-group scale differences make the anti-clustering objective less suitable when distributional properties or higher-order statistics are of greater concern than simple expectations. We demonstrate that the proposed objective in this work is more suitable for scenarios where similarity in distribution or higher-order statistics (beyond simple expectations) is of greater importance.

Rerandomization \cite{sprott1993randomization, 10.1214/12-AOS1008} is another line of work that aims to address imbalanced covariate issues in partitioning control and test subgroups while maintaining the robustness that stems from randomness. In a standard rerandomization approach, a quantification for covariate imbalance is first established along with a threshold for accepting or rejecting subgroup splits. A random partition is then generated, and the covariate imbalance is computed. Based on whether the imbalance falls below the threshold, the partition is either accepted or rejected. The typical imbalance measures used in rerandomization include average difference or Mahalanobis distance, which primarily focus on the average similarity between subgroups but overlook other important distributional discrepancies. Moreover, the threshold for accepting or rejecting a partition is often determined manually, which introduces subjectivity into the process. In contrast, as shown in Section \ref{S:3}, WHOMP can be implemented as a rerandomization strategy that overcomes these limitations. Specifically, it utilizes optimal transport to design an imbalance metric that captures broader distributional discrepancies between subgroups, and it employs unsupervised learning techniques to automatically determine the threshold based on the covariates.

A different, very interesting, approach is taken in~\cite{Spielman2024}. There, the authors first formalize the tradeoff between covariate balance and a notion of robustness. By linking the 
experimental design problem to a new type of problem in algorithmic discrepancy, the authors then propose a randomized algorithm, namely the Gram-Schmidt walk, to solve the distributional discrepancy problem and thereby navigate the tradeoff between balance and robustness.
Their method is limited to the specific setting, where one aims to split the data set into {\em two} subgroups.

Finally, this work is also related to statistical parity in machine learning group fairness. Since a covariate-balanced partition can be viewed as one that is independent of the covariate, its objective aligns with the definition of statistical parity in ML group fairness. In fact, the proposed solution for WHOMP in this work is closely related to performing K-means clustering on an optimal fair data representation that ensures statistical parity. This fair data representation approach guarantees statistical parity for any neutral downstream tasks, such as K-means clustering \cite{JMLR:v24:22-0005}. Here, neutral downstream tasks are the downstream tasks or models that do not introduce statistical dependence on the sensitive information by themselves.

\subsection{Preliminaries and Notation}

We provide a brief review of optimal transport, the Wasserstein-2 space\footnote{As this work focuses on the Wasserstein-2 space, we will henceforth refer to it simply as the Wasserstein space}, and the Wasserstein barycenter, which are essential tools in the development and analysis of WHOMP.

Given $\mu, \nu \in \mathcal{P}(\mathbb{R}^d)$ where $\mathcal{P}(\mathbb{R}^d)$ denotes the set of all the probability measures on $\mathbb{R}^d$, $$\mathcal{W}_2(\mu,\nu) := \left(\inf_{ \lambda \in \prod (\mu,\nu) } \Big\{\int_{\mathbb{R}^d \times \mathbb{R}^d} ||x_1 - x_2||^2 d \lambda(x_1,x_2)\Big\}\right)^{\frac{1}{2}}.$$ Here, $\prod (\mu,\nu) := \{\pi \in \mathcal{P}((\mathbb{R}^d)^2): \int_{\mathbb{R}^d} d\pi(\cdot,v) = \mu, \int_{\mathbb{R}^d} d\pi(u,\cdot) = \nu \}$.
$(\mathcal{P}_2(\mathbb{R}^d),\mathcal{W}_2)$ is called the Wasserstein space, where $\mathcal{P}_2(\mathbb{R}^d):= \Big\{\mu \in \mathcal{P}(\mathbb{R}^d): \int_{\mathbb{R}^d} ||x||^2 d\mu < \infty\Big\}$. To simplify notation, we often denote $$\mathcal{W}_2(X_1,X_2) := \mathcal{W}_2(\mathcal{L}(X_1),\mathcal{L}(X_2)),$$ where $\mathcal{L}(X) := \mathbb{P} \circ X^{-1} \in \mathcal{P}(\mathbb{R}^d)$ is the law or distribution of $X$, $X: \Omega \rightarrow \mathcal{X} := \mathbb{R}^d$ is a random variable (or vector) with an underlying probability space $(\Omega, \mathcal{F}, \mathbb{P})$. Intuitively, one can consider the Wasserstein distance as $L^2$ distance after optimally coupling two random variables whose distributions are $\mu$ and $\nu$. That is, if the pair $(X_1,X_2)$ is an optimal coupling \cite{villani2008optimal}, then $$\mathcal{W}_2(X_1,X_2) = ||X_1 - X_2||_{L^2} = \int_{\Omega} ||X_1(\omega) - X_2(\omega)||^2 d\mathbb{P}(\omega).$$
Given $\{\mu_z\}_{z \in \mathcal{Z}} \subset (\mathcal{P}_2(\mathbb{R}^d),\mathcal{W}_2)$ for some index set $\mathcal{Z}$, their Wasserstein barycenter~\cite{agueh2011barycenters} with weights $\lambda \in \mathcal{P}(\mathcal{Z})$ is 
\begin{equation}\label{barycenter}
\bar{\mu} := \text{argmin}_{\mu \in \mathcal{P}_2(\mathbb{R}^d)} \Big\{\int_{\mathcal{Z}} \mathcal{W}_2^2(\mu_z,\mu) d\lambda(z)\Big\}.
\end{equation}
If there is no danger of confusion, we will refer to the Wasserstein barycenter simply as barycenter.

Two random variables $X_1$ and $X_2$ are called equal in distribution if they have the same probability distribution, which is denoted by $X_1 =_d X_2$. More specifically, $X_1 =_d X_2$ if and only if, for all $f \in \mathcal{C}_b(\mathbb{R}^d)$,
$$    \int_{\mathbb{R}^d} f d\mathcal{L}(X_1) := \int_{\mathbb{R}^d} f(x) d\mathbb{P}\circ(X_1)^{-1}(x) = \int_{\mathbb{R}^d} f(x) d\mathbb{P}\circ(X_2)^{-1}(x) =: \int_{\mathbb{R}^d} f d\mathcal{L} (X_2),
$$
where $\mathcal{C}_b(\mathbb{R}^d)$ denotes the set of all bounded continuous functions on $\mathbb{R}^d$.

\medskip

The rest of this paper is organized as follows: Section 2 defines the Wasserstein Homogeneity Partition (WHOMP) and shows that the WHOMP objective is desirable in group splitting for comparative experiments, such as clinical trials, business A/B tests, and social studies. Section 3 provides a detailed comparison between WHOMP and other partition methods: random partition, stratified randomization, covariate-adaptive randomization, rerandomization, and anti-clustering, which shows the advantage of WHOMP. Section 4 proves a characterization of the WHOMP solution. Section 5 provides an efficient method to estimate the solution to WHOMP which leads to the design of a practical algorithm. Finally, Section 6 demonstrates the advantages of WHOMP via numerical experiments on various data sets.

\section{WHOMP: Wasserstein Homogeneity Partition} \label{S:2}

In this section, we propose the homogeneity partition: given a metric on the probability space, the homogeneity partition aims to minimize the sum of the squared distances between the resulting subgroups and the original data set. Furthermore, we show that, when applying the Wasserstein-2 distance to the homogeneity partition, we can provably minimize the Type I and II error due to the covariate factors resulting from the subgroup partition.

\subsection{Definition of Wasserstein homogeneity}

To start, we define the Wasserstein homogeneity partition for a given data set $X := \{x_i\}_{i \in [N]} \in \mathcal{X}^{[N]}$. To fix the notation below, we let $\mathbf{P}(N,K)$ denote all the partitions on $[N]$ that have $K$ non-empty elements. That is,
\begin{equation}
    P \in \mathbf{P}(N,K) \implies P = \{p_i\}_{i \in [K]} \text{ such that } \bigcup_{i \in [K]} p_i = [N] \text{ and } p_i \cap p_j = \emptyset, \forall i \neq j.
\end{equation}
Also, given a partition $P = \{p_i\}_{i \in [K]}$ on $X$, we define $X_p := \{x_i\}_{i \in p}$ for all $p = p_i \in P$ and $X_P := \{X_p\}_{p \in P} = \{X_{p_i}\}_{i \in [K]}$, which is a set of the $X_p$'s indexed by $[K]$. We also use $p_i$'s to denote the following indicator functions:
\begin{equation*}
p_i(j) = \begin{cases}
1 &\text{ if $j \in p_i$}\\
0 &\text{ if $j \notin p_i$}
\end{cases}
\end{equation*}
Similarly, we often use $P: [N] \rightarrow [K]$ to denote the map that $P(i) = j$ if and only if $x_i \in p_j$. Now, we are ready to define the Wasserstein homogeneity partition:

\begin{defi}[Wasserstein Homogeneity Partition]\label{d:homogeneity_parition}
Given a data set $X := \frac{1}{N} \sum \delta_{x_i} \in \mathcal{P}(\mathcal{X})$ where $N = K \cdot c$ with $K,c \in \mathbb{N}$, the Wasserstein homogeneity partition problem is defined as
\begin{equation} \label{eq:W_homogeneity_partition}
\min_{\substack{Q \in \mathbf{P}(N,c) \\ |q| \equiv K}} \sum_{q \in Q} \mathcal{W}_2^2(X_q, X).
\end{equation}
\end{defi}

Here, since the goal of the partition is to have data splitting for controlled experiments or A/B tests, we require uniform cardinality across the resulting partition subgroups in order to prevent sample ratio mismatch which is a common cause of Simpson's paradox. Moreover, when  N  does not divide K and c, the uniform cardinality constraint can be adjusted in practice or algorithm design to accommodate minimum and maximum cardinality requirements.

One can replace the $\mathcal{W}_2$-distance with some other notion of distance between $X_q$ and $X$ that may be better suited for particular applications.  In this paper, we will focus on $\mathcal{W}_2$ because it is closely related to variance and $L^2$-loss, which leads to straight-forward solutions to WHOMP via K-means clustering.

We will demonstrate in Section~\ref{S:4} that finding the Wasserstein Homogeneity Partition, i.e., computing the solution of~\eqref{eq:W_homogeneity_partition} is actually NP-hard. Yet, the reader need not despair, since our approach to showing NP-hardness also points to a convenient way to sidestep the NP-hardness. Indeed, we will see in Theorem~\ref{th:Homogeneity_Partition_Characterization} that solving~\eqref{eq:W_homogeneity_partition} is equivalent to finding the solution to the (balanced) K-means clustering problem:

\begin{defi}[Balanced K-means Clustering]\label{d:balanced_kmeans}
    Given the data set $X := \frac{1}{N} \sum_{i = 1}^N \delta_{x_i} \in \mathcal{P}(\mathcal{X})$, the Balanced K-means Partition problem is defined as
    \begin{equation} \label{eq: balanced_kmeans}
        P \in \argmin_{\substack{P \in \mathbf{P}(N,K), \\ |p| \equiv c}} \sum_{p \in P} \var(X_p),
    \end{equation} 
\end{defi}
Based on this connection we will later present an approximate solution to Problem~\eqref{eq:W_homogeneity_partition} by combining a balanced K-means approximation algorithm and randomness.

\subsection{Theoretical guarantees for WHOMP}

Here, we show that the objective function of the proposed homogeneity partition makes it suitable for comparative experiments such as clinical trials, A/B tests, and social science studies: The objective tends to penalize the average distributional discrepancy (quantified by the Wasserstein distance) between the subgroups and the original sample data and, hence, minimize the influence of the subsampling process on the controlled experiments outcome.

In particular, the following results provide a comprehensive justification for the WHOMP objective as a natural subsampling approach for statistical tests in comparative experiments, presented from three perspectives: qualitative (Theorem \ref{th:no_type_error}), concrete (Example \ref{exam:variance_homogeneity} and the results therein), and quantitative (Corollary \ref{co:lipschitz_error_bound} and Theorem \ref{th:error_bound_repeat_trial}). To fix ideas for the following results, let $X$ be the available feature variables in the sample, $C$ the control variable, $Q$ the subsampling assignment variable, and $Y$ the controlled experiment outcome. To simplify notation, we assume the experiment will apply the control factor $c_i$ to the subgroup $X_{q_i}$ for all $i \in [c]$. 

To start, we demonstrate that a zero WHOMP objective eliminates Type I and Type II errors arising from the subgroup splitting variable  $Q$  in the context of statistical or social experiments, which motivates the use of the Wasserstein homogeneity criterion:

\begin{thm}[No type I or type II error due to subgroup]\label{th:no_type_error}
    Assume that, for all $q \in Q$, $\mathcal{W}_2^2(X_q,X) = 0$. Also, let $Y: \mathcal{X} \times \mathcal{C} \rightarrow \mathcal{X}$ be the true outcome, which we assume to be an arbitrary measurable function. Then the following holds:
    \begin{itemize}
        \item For all $c_0,c_1 \in \mathcal{C}$ satisfying $Y(X,c_1) =_d Y(X,c_0)$, $Y(X_{q_1},c_1) =_d Y(X_{q_0},c_0)$
        \item For all $c_0,c_1 \in \mathcal{C}$ satisfying $Y(X,c_1) \neq_d Y(X,c_0)$, $Y(X_{q_1},c_1) \neq_d Y(X_{q_0},c_0)$
    \end{itemize}
\end{thm}

\begin{proof}
    See Appendix \ref{A:2}. 
\end{proof}

The above result shows that, for any ineffective treatment on the population $X$, the control/test experiment under $(X_{q_0},X_{q_1})$-splitting reveals the ineffectiveness truthfully. In other words, no type I error arises from subgroup partitioning or the subgroup variable $Q$. Similarly, for any treatment that is genuinely effective for the population, the control/test experiment under $(X_{q_0}, X_{q_1})$ partitioning must demonstrate this effectiveness. Specifically, there exists a test function $f \in \mathcal{C}_b(\mathcal{X})$ such that, when conducting hypothesis testing on the average effect of $f$, no type II error is introduced by the subgroup variable due to group-splitting. Although, for ease of presentation, Theorem \ref{th:no_type_error} is stated in the case where the subgroups and controls are binary, it is clear that our problem setting is suitable for arbitrary discrete or continuous subgroups or control variables.

Next, we demonstrate the effect of zero WHOMP objective in hypothesis testing for comparative experiments via the following example:

\begin{exam}[Theorem \ref{th:no_type_error} in Hypothesis Testing]
\normalfont
Here, we use a standard hypothesis testing setting to illustrate Theorem~\ref{th:no_type_error} above and motivate the Theorem \ref{th:error_bound_repeat_trial} below, under the assumption that the experimenter is interested in estimating the average treatment effect:
\begin{equation}
    \tau := \frac{1}{n}\sum_{i = 1}^n y_i(0) - \frac{1}{n}\sum_{i = 1}^n y_i(1),
\end{equation}
where $y_i(j) := Y(x_i,c_j) \text{ for } j \in \{0,1\}$. In experiments, it is not possible to observe the effects of different controlled factors on the same $x_i$. Therefore, we employ the following natural estimator to approximate $\tau$:
\begin{equation}
    \hat{\tau} := \bar{Y}_{\text{obs}}(0) - \bar{Y}_{\text{obs}}(1),
\end{equation}
where $\bar{Y}_{\text{obs}}(0) := \frac{\sum_{i} y'_i(0)\mathcal{Q}(i)}{\sum_{i} \mathcal{Q}(i)}$ and $\bar{Y}_{\text{obs}}(1) := \frac{\sum_{i} y'_i(1)[1 - \mathcal{Q}(i)]}{\sum_{i} [1-\mathcal{Q}(i)]}$. Here, $\mathcal{Q}: [N] \rightarrow \{0,1\}$ represents the randomized partitions drawn from $\mathcal{Q}(P)$ (Definition \ref{d:Q(P)}) or, equivalently, resulting from WHOMP Random (see definition in Algorithm \ref{algo: WHOMP_random}, Section \ref{S:5}).

The following result demonstrates that WHOMP Random produces an unbiased estimator:

\begin{lem}[$\hat{\tau}$ is an unbiased estimator of $\tau$] \label{l:unbias_estimator}
\begin{equation}
    \mathbb{E}(\hat{\tau}) = \tau.
\end{equation}
\end{lem}
See proof in Appendix \ref{A:2}. Now, we proceed with the standard steps for hypothesis testing:
\begin{itemize}
    \item[(i)] \textit{Null Hypothesis}: Assume the null hypothesis of no average treatment effect, expressed as $\tau = 0$ because $y_i(0) := Y(x_i, c_0) = Y(x_i, c_1) =: y_i(1), \forall i \in [n]$.
    \item[(ii)] \textit{Null Distribution}: Generate the empirical estimator or null distribution, defined by the law of $\hat{\tau}(\mathcal{Q})$: $\mathcal{L}(\hat{\tau}(\mathcal{Q})|\mathcal{Q}(P)) := \{\frac{\sum_{i} y_i(0)\mathcal{Q}(i)}{\sum_{i} \mathcal{Q}(i)} - \frac{\sum_{i} y_i(1)[1 - \mathcal{Q}(i)]}{\sum_{i} [1-\mathcal{Q}(i)]}, \mathcal{Q} \sim \mathcal{Q}(P)\} $.
    \item[(iii)] \textit{p-Value}: Compute the p-value as the frequency of equally or larger absolute value in $\mathcal{L}(\hat{\tau}(\mathcal{Q}))$ compared to the absolute value of the actual experimental observation.
\end{itemize}

The following result demonstrates how Theorem \ref{th:no_type_error} applies within this hypothesis testing framework:
\begin{coro}[Zero-One p-value] \label{coro:zero_one_p_value}
    If $\mathcal{W}_2^2(X_i, X) = 0$ for $i \in \{0,1\}$, it follows that
    \begin{equation}
        \mathcal{L}(\hat{\tau}) = \delta_{\tau} = \delta_0.
    \end{equation}
    In other words, the p-value is either $0$ or $1$.
\end{coro}

\begin{proof}
Since $X$ and $Q$ are finite, let $L := \max_{x,x' \in \mathcal{X}} \max_{q,q' \in Q} \min \{L: ||Y(x,c_q) - Y(x',c_{q'})||_{l^2} \leq L ||x - x'||_{l^2}  \}$. Notice the $L$ is well-defined due to the null hypothesis that $Y(x,c_q) = Y(x,c_{q'}), \forall q, q' \in Q, \forall x \in X$. In addition, it follows from the null hypothesis and Lemma \ref{l:unbias_estimator} that $\tau = 0$. Finally, it follows from Theorem \ref{th:error_bound_repeat_trial} that, for all $\epsilon > 0$, we have
\begin{equation}
    \Var(\hat{\tau})  = \mathbb{E}(||\hat{\tau} - \tau||^2_{l^2})
    \leq L^2 \frac{|Q|}{|Q|-1} \sum_{q \in Q} \mathcal{W}_2^2(X_q,X)
    < \epsilon.
\end{equation}
Therefore, we have $\Var(\hat{\tau}) = 0$, which implies $\mathcal{L}(\hat{\tau}) = \delta_0$.
\end{proof}

Intuitively, the WHOMP objective is an effective tool for controlling the concentration (or standard deviation) of the null distribution around the true average treatment effect, especially when the observations are statistically dependent on the given covariates. Such control of concentration is crucial in scenarios where experiments can only afford to repeat the random trial a limited number of times. In the case described, measurability (with respect to $(X,C)$) and a zero WHOMP objective lead to the null distribution collapsing into a Dirac measure. See Theorem \ref{th:error_bound_repeat_trial} for a more detailed result on the control of the distributional concentration of $\hat{\tau}$, under the more general assumptions that $Y$ is Lipshcitz with respect to $X$.

One can also perform other hypothesis tests using the Wasserstein distance instead of the average distance. Specifically, one can use $\mathcal{W}_2^2(Y(X_0,c_0), Y(X_1,c_1))$ as an estimator for $\mathcal{W}_2^2(Y(X,c_0), Y(X,c_1))$. It can be shown that
\begin{equation}
    \mathbb{E}(\hat{\tau}) = \mathbb{E}(\mathcal{W}_2^2(Y(X_0,c_0), Y(X_1,c_1))) = 4\mathbb{E}(\Var(X_P)) \neq 0 = \tau.
\end{equation}
Therefore, $\mathbb{E}(\Var(X_P))$, which is proportional to the WHOMP objective, determines the bias in this case. For further details on hypothesis testing based directly on Wasserstein variation, we refer interested readers to \cite{munk1998nonparametric, czado1998assessing, del2019central} and the references therein.
\end{exam}

Next, we demonstrate how the WHOMP objective bound the statistics estimation error between the resulting subgroups and the original sample:

\begin{coro}[Lipschitz Statistics Error Bound] \label{co:lipschitz_error_bound}
    Assume $\frac{1}{|Q|} \sum_{q \in Q} \mathcal{W}_2^2(X,X_q) \leq d$ for some $d \geq 0$, then for any $\epsilon > 0$, and $h: \mathcal{X} \rightarrow \mathbb{R}$, we have $$\mathbb{P}( \{ \sup_{||h||_{Lip} \leq L} |\mathbb{E}(h(X)) - \mathbb{E}(h(X_q))| > \epsilon \} ) \leq \frac{L\sqrt{d}}{\epsilon}.$$
\end{coro}

\begin{proof}
    See Appendix \ref{A:2}.
\end{proof}

Here, $||h||_{Lip} \leq L$ means that $h$ is $L$-Lipschitz for some $L \in \mathbb{R}_{+}$. The above result shows that if the objective of Problem \eqref{eq:W_homogeneity_partition} (averaged by $|Q|$) is bounded by some $d$ that is relatively small compared to $\frac{\epsilon}{L}$ for the chosen $\epsilon$, then it is unlikely to observe that any $L$-Lipschitz statistics on $X$ and $X_q$ differ by more than $\epsilon$.

Finally, we provide a quantitative version of Corollary \ref{coro:zero_one_p_value}: How the WHOMP objective controls the concentration of the average treatment effect estimator (unbiased by Lemma \ref{l:unbias_estimator}) around the true average effect?

\begin{thm}[Variance Bound for Average Treatment Effect Estimator]\label{th:error_bound_repeat_trial}

    Assume the observation is a uniformly (with respect to the control factor $\{c_q\}_{q \in Q}$) Lipschitz function of the given covariate $X$:
    \begin{equation}
        \sup_{q,q' \in Q} ||Y(x,c_q) - Y(x',c_{q'}) ||_{l^2} \leq L ||x - x'||_{l^2}, \forall x,x' \in \mathcal{X},
    \end{equation}
    then we have
    \begin{equation}
        \mathbb{E}(||\hat{\tau} - \tau||^2_{l^2}) \leq L^2 \frac{|Q|}{|Q|-1} \sum_{q \in Q} \mathcal{W}_2^2(X_q,X), \forall Q \in \mathcal{Q}(P).
    \end{equation}
    
\end{thm}

\begin{proof}
    See Appendix \ref{A:2}
\end{proof}

Without further assumptions on $X$, the inequality above is sharp as equality can be achieved by Gaussian mixture models.

\section{Comparison with Related Subsampling Partition Methods}\label{S:3}

In this section, we provide an analysis of random subsampling, covariate-adaptive randomization, rerandomization, and anti-clustering, compared with WHOMP.

\subsection{Random Sampling}

Now, we show that pure random subsampling can result in large distributional deviations from the original sample, especially in the case of small subsample sizes. It is easy to construct examples with specific assumptions on sample distribution. For example, given a linear model $Y = ax + N$ where $N \sim \mathcal{N}(b,\sigma^2)$ is the Gaussian noise, one can consider subsamples as i.i.d. random variables drawing from $Y$ and thereby conclude it is not unlikely to observe subsamples that significantly differ from $Y$.

Here, we give a result on the subsample deviation in terms of Wasserstein distance without assuming the sample distribution. Instead, we assume a quantity that we will need in our main result to characterize the solution to problem \ref{prop:clustering_random_duality}. In particular, we first show a deterministic lower bound when the subsample size is small, which is sharp and closely related to the Theorem \ref{th:Homogeneity_Partition_Characterization} below, and thereby show that it is not unlikely to obtain large distributional deviations.

\begin{lem}[Subsample Wasserstein Deviation Lower Bound] \label{l:random_sample_deviation_bound}
Let $X = \{x_i\}_{i = 1}^N$ be a sample data set and $X_{\text{sample}} := \{x_i\}_{i = 1}^K$ a subsample, where $x_i \sim uniform(X)$ are i.i.d. random variables sampled from $X$ with uniform distribution. Then
\begin{equation} \label{eq:supplement_1}
    \mathcal{W}_2^2(X,X_{\text{sample}}) \geq \min_{\substack{P \in \mathbf{P}(N,K) \\ |p| \equiv c}} \frac{1}{K} \sum_{p \in P} \var(X_p).
\end{equation}
\end{lem}

\begin{proof}
    See Appendix \ref{A:3}.
\end{proof}

Without further assumptions on the distribution of $X$, the lower bound is sharp due to the optimal partition $P$ definition. But now we show that this lower bound is very unlikely to be obtained via random partition in practice and one should expect a much higher Wasserstein deviation with high probability:

\begin{coro}[Distribution of Subsample Wasserstein Deviation Lower Bounds]
With probability $1-\frac{K!}{K^K}$, we have
\begin{equation} \label{eq:supplement_2}
    \mathcal{W}_2^2(X,X_{\text{sample}}) \geq \min_{\substack{P \in \mathbf{P}(N,K) \\ |p| \equiv c}} \frac{1}{K} \sum_{p \in P} \var(X_p) + \min_{p \neq p'} ||\bar{X}_p - \bar{X}_{p'}||_2^2.
\end{equation}
\end{coro}

The distribution of the lower bounds is equal to drawing $\{I_k\}_{k = 1}^K$ i.i.d. from $[K]$ uniformly with replacement and then sum $\min_{\substack{P \in \mathbf{P}(N,K), \\ |p| \equiv c}} \frac{1}{K} \sum_{p \in P} \var(X_p)$ and \begin{equation}
    \min_{\sigma} \sum_{ \substack{i: I_k \neq i, \forall k \\ j: |\{k: I_k = j\}| > 1} } ||\bar{X}_{p_i} - \bar{X}_{p_{\sigma(j)}} ||_2^2.
\end{equation}

Last but not least, it is important to note that the strength of randomized subsampling lies in its ability to enhance robustness and generalizability, which is why it is widely used in machine learning training and testing. However, this advantage can actually weaken the result of controlled experiments.

\begin{rema}[An Objective Perspective on Controlled Experiment Partitioning]

The primary objective of a comparative experiment is to apply different controlled factors to distinct subgroups and assess their effectiveness by comparing the conditional outcome distributions. Therefore, the subsampling process should aim to produce subgroups that closely resemble the original sample’s distribution, minimizing the risk of type I and type II errors from imbalanced splits.

In contrast, the advantage of randomization in statistical tests arises from two key aspects: the simplicity due to the law of large numbers (via subsample size, repeated trials, or both) and the improvement of model robustness and generalizability. Randomized subgroups effectively capture potential differences between the sample data and the true population.

However, such distributional discrepancies weaken the outcomes of controlled experiments. For instance, in hypothesis testing where the null hypothesis posits that the controlled factor has no effect, and the alternative suggests it is effective, any distributional shifts between subgroups must be incorporated into the null distribution. This adjustment narrows the rejection region, reducing the test’s power against the alternative hypothesis.

Thus, in controlled experiments where the law of large numbers is less applicable and robustness is not the primary concern, subsampling should prioritize forming subgroups that closely replicate the original sample’s distribution. This strategy enhances the statistical power of the experimental results.
\end{rema}

\subsection{Covariate Adaptive Randomization}

A general framework of covariate adaptive randomization aims to design a group assignment strategy that minimizes an imbalance score, which involves the selected covariates and random treatment membership assignment. But there are at least the following disadvantages: (1) manual discretization of continuous covariates, (2) lack of optimality criteria related to comparative test performance guarantee, and (3) a lack of theoretical guarantee of the similarity among the resulting sensitive groups concerning either the selected covariates or other feature variables in the data set.

For example, both stratified randomization and the basic version of minimization aim to balance the cardinality of the members in subgroups with respective to some covariates. Here, the imbalance score is defined as the difference in the number of members sharing the same (discretized) covariate value across different subgroups.

A more general framework of covariate adaptive randomization replaces the cardinality imbalance score with the following $$\text{Imbalance} := ||\sum_{i = 1}^n (2T_i - 1) f(X_i)||^2,$$ where $T_i$ is the random subgroup assignment, $X_i$ denotes the selected covariates, and $f$ is a function to take care of higher order statistics of the covariates. But such a framework still focuses on the difference between the average: such an imbalance score is based on the {\em difference between the averages} of the subgroups. Therefore, it largely ignores the distributional differences across the subgroups.

In comparison, WHOMP also minimizes an imbalance score. But the score is now an {\em average of the differences} between the optimally matched or coupled members. By switching the order of difference and averaging by leveraging the coupling (w.r.t. all the feature variables in the data set) technique from optimal transport, the WHOMP objective can capture the distributional differences (w.r.t. all the feature variables in the data set) across the subgroups.

Therefore, WHOMP could also be considered a covariate adaptive randomization technique, albeit one that comes with an optimality criterion that provides provable guarantees for comparative test performance and improved practical outcomes.

\subsection{Rerandomization}

We will demonstrate that WHOMP can be viewed as a rerandomization method while at the same time addressing two key shortcomings of traditional rerandomization methods. Specifically, traditional rerandomization methods rely on mean discrepancies to quantify covariate imbalances. However, as discussed earlier, mean discrepancies capture only limited aspects of distributional differences, often missing higher-order statistical discrepancies in the data. Additionally, conventional rerandomization methods typically require a manually selected threshold for determining whether a partition is acceptable, which introduces subjectivity.

The following discussion shows that WHOMP resolves the two issues by leveraging optimal transport and unsupervised learning techniques:

\begin{itemize}
    \item \textit{Comprehensive Imbalance Quantification}: WHOMP utilizes Wasserstein distance to design a covariate imbalance metric that not only accounts for mean differences but also captures a broader range of distributional discrepancies between subgroups. By switching the order of taking differences and averaging, the WHOMP distributional discrepancy metric ensures a more thorough and nuanced assessment of subgroup similarity.
    \item \textit{Self-learned Threshold}: WHOMP replaces the manual threshold selection with an automated threshold learned by balanced K-means directly from the data based on the requested subgroup number. This self-learned threshold reduces subjectivity and makes the method more objective and consistent across different datasets.
\end{itemize}

\begin{lem}[Equivalence between WHOMP and Rerandomization] \label{l: rerandomization_equivalence}
    Assume that the balanced K-means problem (Definition \ref{d:balanced_kmeans}) has a unique solution, denoted by $P$. Define the subgroup splitting accept and reject rule $\Phi(X,Q)$ by
    \begin{equation}
        \Phi(X,Q) := \mathds{1}_{ \{\Var(\bar{X}_Q) = \Var(\mathbb{E}(X_P))\} }(Q).
    \end{equation}
    Then  WHOMP Random (Algorithm \ref{algo: WHOMP_random}) is equivalent to rerandomization with $\Phi(X,Q)$.
\end{lem}

\begin{proof}
    See Appendix \ref{A:3}
\end{proof}

\subsection{Anti-clustering}

We first review anti-clustering, then provide a provable and efficient estimation method for anti-clustering, and finally show the problem of scale difference across subgroups in anti-clustering and how WHOMP solves the problem.

Anti-clustering was first introduced in \cite{spath1986anticlustering}. The name comes from the fact that the objective is the exact opposite of the classic K-means clustering objective. In particular, given a data set $\{x_i\}_{i = 1}^N$, K-means has the following objective
\begin{equation*}
\min_{P \in \mathbf{P}(N,K) } \sum_{p \in P} \sum_{x \in X_p} \frac{1}{|p|} ||x - \mathbb{E}(X_p)||_2^2 \equiv \min_{P \in \mathbf{P}(N,K) } \sum_{p \in P} \var(X_p)
\end{equation*}
In contrast, anti-clustering is an optimization problem that has the opposite objective of K-means:
\begin{defi}[Anti-clustering \cite{spath1986anticlustering}] For a fixed $c \in [N]$, the following optimization problem is call anti-clustering with partition cardinality $c$:
    \begin{equation}
        \max_{Q \in \mathbf{P}(N,c)} \sum_{q \in Q} \sum_{x \in X_q} ||x - \mathbb{E}(X_q)||_2^2
    \end{equation}
\end{defi}

The following result shows that, when restricting the partition to have uniform cardinality across subgroups, the expected value of the anti-clustering objective via a random selection across partition elements coincides with the objective of the (balanced) K-means objective:

\begin{lem}[Random selection duality for anti-clustering]\label{l:anticlustering_random_duality}
\begin{equation} \label{eq:anticlustering_random_duality}
\min_{\substack{P \in \mathbf{P}(N,K) \\ |p| \equiv c}} \sum_{p \in P} \Var(X_p) \iff \max_{\substack{P \in \mathbf{P}(N,K) \\ |p| \equiv c}} \mathbb{E}_{\mathcal{Q} \sim uniform(\mathcal{Q}(P)) }[\sum_{q \in \mathcal{Q}} \Var(X_q)]
\end{equation}
\end{lem}

\begin{proof}
    See Appendix \ref{A:3}.
\end{proof}

In short, if we create a partition $\mathcal{Q} \sim uniform(\mathcal{Q}(P))$ via the random selection method across subgroups in a given $P$, then the expected anti-clustering objective for the randomly selected $\mathcal{Q}$ (RHS in \eqref{eq:anticlustering_random_duality}) is equivalent to the K-means objective for $P$ (LHS in \eqref{eq:anticlustering_random_duality}).

For theoretical interest, we also show the following result which is the counterpart of Lemma $\ref{l:anticlustering_random_duality}$ above.

\begin{prop}[Random selection duality for clustering]\label{prop:clustering_random_duality}
\begin{equation}
    \max_{\substack{Q \in \mathbf{P}(N,c) \\ |q| \equiv K }} \sum_{q \in Q} \Var(X_q)  \iff \max_{\substack{Q \in \mathbf{P}(N,c) \\ |q| \equiv K }} \mathbb{E}_{\mathcal{P} \sim uniform(\mathcal{P}(Q)) }[\sum_{p \in \mathcal{P}} \Var(X_p)]
\end{equation}
\end{prop}

\begin{proof}
    See Appendix \ref{A:3}
\end{proof}

Here, the setting is similar to the setting above, except that we fix $Q$ first, then define $\mathcal{P}(Q)$, and generate random partition $\mathcal{P} \sim uniform(\mathcal{P}(Q))$.

Interestingly, as we will show later in Section \ref{S:4}, this random selection from the K-means approach coincides with the optimal solutions to WHOMP. That is, the random estimation of anti-clustering outperforms the exact solution to anti-clustering in terms of the WHOMP objective. Now, we illustrate the disadvantages of the anti-clustering objective and its exact solution compared to WHOMP solutions.

In practice, diversity within each subgroup resulting from partitioning is desirable, as a more diverse subgroup better captures the overall structure of the sample data. This rationale underlies the anti-clustering objective: maximizing the sum of variance within subgroups to promote diversity.

While diversity is advantageous, scale differences across subgroups are undesirable, particularly in mean squared error (MSE) or  $L^2$  loss scenarios. For instance, comparing MSE loss in cross-validation or hold-out sets is problematic when training and test datasets differ by a scale factor, even if their data structures are otherwise similar.

Now, we show that anti-clustering tends to produce subgroups at different scales. To start, we need the following characterization of the anti-clustering objective.

\begin{lem}[Centroid Variable Characterization, \cite{spath1986anticlustering}]\label{centroid variable characterization}
\begin{equation}
\max_{P \in \mathbf{P}(N,K)} \sum_{p \in P} \sum_{x \in X_p} ||x - \mathbb{E}(X_p)||_2^2 \iff \min_{P \in \mathbf{P}(N,K)} \sum_{p \in P} |p| ||\mathbb{E}(X_p) - \mathbb{E}(X)||_2^2
\end{equation}
\end{lem}

Intuitively, an anti-clustering partition generates diversity within the resulting subgroups by enforcing low variance among the subgroup centroids. However, the enforcement of low variance among the subgroup centroids leads to scale differences across the resulting subgroups for the following reason: To have similar centroids or means for subgroups, data points sharing the same scale tend to be group together to balance each other so that centroids of subgroups can stay as close as possible.

Now, we use an example to show that anti-clustering leads to diversity in terms of variance or structure across the elements due to the enforced homogeneity of the centroids.

\begin{exam}[Variance scale differences across anti-clustering subgroups]\label{exam:variance_homogeneity}

Given the data points $\{x_i\}_{i \in [9]}$ as shown in Figure \ref{anti vs desire}, the left plot uses the dash lines to connect data points that belong to the same element of an anti-clustering partition. In this case, we have the anti-clustering partition is the following: $$\mathcal{P}_{anti-clustering} = \{\{x_1,x_5,x_9\},\{x_2,x_6,x_7\}, \{x_3,x_4,x_8\}\}.$$ such a partition is guaranteed to be the anti-clustering because the right-hand side of the centroid variable characterization is zero and therefore minimized.

\begin{figure}[H]
\centering
\includegraphics[width=\textwidth]{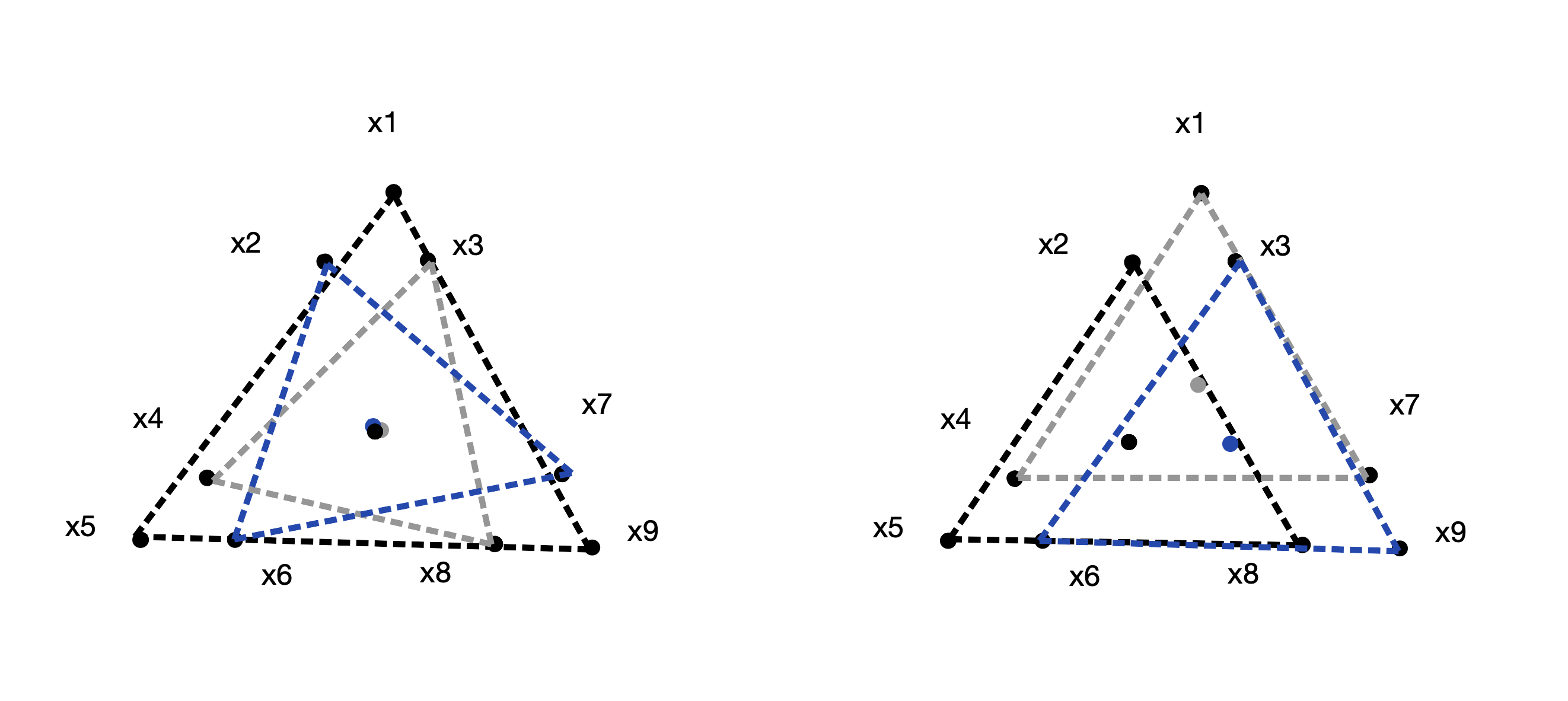}
\caption{This example illustrates that anti-clustering (left) tends to produce subgroups at different scales: compare the size of the larger triangle formed by $x_1, x_5, x_9$ with the size of the smaller triangles formed by $x_2, x_6, x_7$ and $x_3, x_4, x_8$, respectively. In comparison, WHOMP Matching (right) leads to the desired subgroup partition at the same scale.}
\label{anti vs desire}
\end{figure}

On the other hand, if we hope to not only obtain diversity within each of the partition elements but also a similarity in structure, scale, or variance across the elements. Then the desired partition would be the following: $$\mathcal{P}_{desired} = \{\{x_1,x_4,x_7\},\{x_2,x_5,x_8\}, \{x_3,x_6,x_9\}\}.$$ As shown in the figure, the right plot uses dash lines to connect data points that belong to the same desirable partition element in this example.
\end{exam}

\section{Overcoming the NP-hardness of Wasserstein Homogeneity Partition}\label{S:4}

In this section, we first characterize the set of all optimal solutions to the Wasserstein Homogeneity Partition Problem (WHOMP) by utilizing solutions to the balanced K-means problem (Definition \ref{d:balanced_kmeans}). Using this characterization, combined with the estimated balanced K-means solution, we propose an efficient approach for computing approximate WHOMP solutions. We also highlight a trade-off between the homogeneity of the mean and variance among the optimal solutions, illustrating how ``anti-clustering'' and the Wasserstein barycenter represent the two extremes of this trade-off.

Since the balanced K-means is a special case of the 2-norm clustering problem with cardinality constraints, it is known to be NP-hard \cite{bertoni2012size}. Therefore, the provable equivalence between the balanced K-means and WHOMP implies that the proposed WHOMP problem is also NP-hard. To mitigate this complexity, we employ multiple random initialization, the fact that K-means is equivalent to the Wasserstein barycenter problem, and a constrained K-means clustering algorithm inspired by \cite{bradley2000constrained} to estimate the balanced K-means solution. It is worth noting that, while our approach uses the constrained K-means algorithm for implementation, any alternative balanced K-means estimation method could be employed instead.

\subsection{Characterization of WHOMP Solutions}

We first derive a characterization of the solution to~\eqref{eq:W_homogeneity_partition} and then apply this characterization to design an algorithm to construct the WHOMP partition. To simplify notation in the rest of this section, given a partition $P \in \mathbf{P}(N,c)$, we denote the {\em Wasserstein barycenter} of $\{X_p\}_{p \in P}$ by $\Bar{X}_P := \bary(\{X_p\}_{p \in P})$. Also, we denote $\mathbb{E}(X_P) = (\mathbb{E}(X_p))_{p \in P}$ to be the vector of the centroids of $X_P$. Finally, for a partition $P \in \mathbf{P}(N,c)$ we define the partitions resulting from selection from $P$ as follows:

\begin{defi}[Partition Selected from $P$: $\mathcal{Q}(P)$]\label{d:Q(P)}
Given bijective maps $\{T_p\}_p$ where each $T_p$ map $\mathcal{L}(\Bar{X}_P)$ to $\mathcal{L}(X_p)$, we construct a partition 
\begin{equation}
    Q(\{T_p\}_p) := \{q(\Bar{x})\}_{\Bar{x} \in \Bar{X}_P}
\end{equation}
with $q(\Bar{x}) := \{i: x_i \in \cup_{p \in P} \{T_p(\Bar{x})\}$. We define $\mathcal{Q}(P)$ to be the set of all the partitions of the above form:
\begin{equation}
    \mathcal{Q}(P) := \bigcup_{\{T_p\}_p: T_p \text{ is bijective}} \{Q(\{T_p\}_p): {T_p}_{\sharp} \mathcal{L}(\Bar{X}_P) = \mathcal{L}(X_p)\}
\end{equation}
\end{defi}

Now, we are ready to state the main result of this section:

\begin{thm}[Wasserstein Homogeneity Partition Characterization]\label{th:Homogeneity_Partition_Characterization}
    Given $P$ a solution to the K-means partition under the uniform cardinality constraint:
    \begin{equation}
        P \in \argmin_{\substack{P \in \mathbf{P}(N,K), \\ |p| \equiv c}} \sum_{p \in P} \var(X_p),
    \end{equation} 
    then $\mathcal{Q}(P)$ are the set of solutions to \eqref{eq:W_homogeneity_partition}:
    \begin{equation}
        \mathcal{Q}(P) = \argmin_{\substack{Q \in \mathbf{P}(N,c), \\ |q| \equiv K}} \sum_{q \in Q} \mathcal{W}_2^2(X_q, X).
    \end{equation}
\end{thm}

To prove the above result, we need the following two lemmas, which are also of independent interest.

\begin{lem}[Maximum Variance Barycenter Characterization of \eqref{eq:W_homogeneity_partition}]\label{l:barycenter_variance_characterization}
    \begin{equation}
        \min_{\substack{Q \in \mathbf{P}(N,c), \\ |q| \equiv K}} \sum_{q \in Q} \mathcal{W}_2^2(X_q,X) \equiv \max_{\substack{Q \in \mathbf{P}(N,c), \\ |q| \equiv K}} \var(\Bar{X}_Q).
    \end{equation}
\end{lem}

See proof in Appendix \ref{A:4}. That is to say, for any fixed partition cardinality $c \in \mathbb{N}$, the subgroups resulting from a Wasserstein homogeneity partition of cardinality $c$ satisfy that their Wasserstein barycenter has the largest variance among all possible partitions of cardinality $c$ with element-wise uniform cardinality constraint. On the other hand, if one obtains a partition satisfying the Wasserstein barycenter of the resulting subgroups which has a larger variance than any other subgroup barycenters resulting from partitions with cardinality $c$, then the partition is a solution to the Wasserstein homogeneity partition.

Therefore, to prove Theorem \ref{th:Homogeneity_Partition_Characterization}, it suffices to show that any $Q \in \mathcal{Q}(P)$ results in a $\Bar{X}_Q$ with larger or equal variance than any other partition with cardinality $c$. To that end, we need the following result that relates $\Bar{X}_Q$ to $\mathbb{E}(X_P)$:

\begin{lem}[Barycenter of $X_Q$ Equals Centroids of $X_P$ for All $Q \in \mathcal{Q}(P)$]\label{l:random_barycenter}
    For all $Q \in \mathcal{Q}(P)$, we have 
    \begin{equation}
        \mathcal{W}_2(\Bar{X}_Q, \mathbb{E}(X_P)) = 0.
    \end{equation}
\end{lem}

See proof in Appendix \ref{A:4}. Finally, we are ready to prove Theorem~\ref{th:Homogeneity_Partition_Characterization} by leveraging the two lemmas above.\\

\begin{proof}[Proof of Theorem \ref{th:Homogeneity_Partition_Characterization}]
    Assume for contradiction that there is another $Q' \in \mathbf{P}(N,c) \setminus \mathcal{Q}(P)$ such that $q' \equiv K$ and $\var( \Bar{X}_{Q'}) > \var(\Bar{X}_{Q})$. Let $T_{q'}$ be the optimal transport maps from $\Bar{X}_{Q'}$ to $X_{q'}$ for all $q' \in Q'$. Now, define $p'(\Bar{x}) := \{T_{q'}(\Bar{x})\}_{q' \in Q'}$ for each $\Bar{x} \in \Bar{X}_{Q'}$ and $P' := \{p'(\Bar{x})\}_{\Bar{x} \in \Bar{X}_{Q'}}$. It then follows from $|q'| \equiv K$ that $|\Bar{X}_{Q'}| = K$. Therefore, we have $P' \in \mathbf{P}(N,K)$ and $|p'(\Bar{x})| \equiv \frac{N}{K} = c$ by construction. Also, for each $\Bar{x} \in \Bar{X}_{Q'}$, we have
    \begin{equation}
        \mathbb{E}(X_{p'}) = \frac{1}{|Q'|} \sum_{q' \in Q'} T_{q'}(\Bar{x}) = \Id(\Bar{x}) = \Bar{x}.
    \end{equation}
    It follows that
    \begin{equation}
        \var(\mathbb{E}(X_{P'})) = \var(\Bar{X}_{Q'}) > \var(\Bar{X}_Q) = \var(\mathbb{E}(X_{P})).
    \end{equation}
    Here, the last equality follows from Lemma \ref{l:random_barycenter}. But this contradicts the optimality of $P$. Now, for any $Q, Q' \in \mathcal{Q}(P)$, we have
    \begin{equation}
        \var(\Bar{X}_{Q'}) = \var(\mathbb{E}(X_{P})) = \var(\Bar{X}_Q).
    \end{equation}Hence, we have proved by contradiction that each $Q \in \mathcal{Q}(P)$ satisfies $\var(\bar{X}_Q) \geq \var(\bar{X}_{Q'})$ for all $Q'$ that satisfy $Q' \in \mathbf{P}(N,c)$ and $q' \equiv K$. Finally, it follows from Lemma \ref{l:barycenter_variance_characterization} that $Q$ is a solution to \eqref{eq:W_homogeneity_partition}. The proof is complete.
\end{proof}

\subsection{Mean and Variance Trade-off among Optimal Solutions}

Since the optimal solution to WHOMP is not unique, we are naturally led to the question of how the various optimal solutions differ from each other. To that end, we show a trade-off in variance between the first two moments among the optimal solutions to WHOMP. Furthermore, we show that, among the trade-offs, the extremal solution that minimizes the variance of the subgroup's first moments has the most ``anti-clustering'' characteristic, while the extremal solution that minimizes the variance of the subgroup's second moments (Algorithm \ref{algo:WHOMP_matching}) is closely connected to the Wasserstein barycenter; whereas the randomized WHOMP solutions  (see details of WHOMP Random design in Algorithm \ref{algo: WHOMP_random}) tend to achieve a balance between the two extremes.

Due to the disadvantages of anti-clustering pointed out in Section \ref{S:3} above, although one can adopt anti-clustering methods to approximate that extremal solution, we focus on the other extreme of the trade-off and balanced solutions in between. Now, we first derive the trade-off.

\begin{lem}[Averages variance and variances average trade-off in $\mathcal{Q}(P)$]\label{l:subsample_BV_tradeoff}
Given a partition $X_P$ on the data $X$ with $P \in \mathbf{P}(N,K)$ and $Q \in \mathcal{Q}(P)$, it follows that
\begin{equation}
    \var(X) =  \underbrace{\var(\mathbb{E}(X_Q))}_{\text{variance of subgroup expectations}} + \underbrace{ \frac{1}{|Q|} \sum_{q \in Q} \var(X_q)}_{\text{expected in-subgroup variance}}.
\end{equation}
\end{lem}

\begin{proof}
It follows directly from law of total variance with $\mathbb{E}(X_Q) = \mathbb{E}(X|Q)$ and $$\frac{1}{|Q|} \sum_{q \in Q} \var(X_q) = \mathbb{E}(\var(X|Q)).$$
\end{proof}

The above result shows that, when choosing among the optimal solutions to WHOMP, one can choose either the ones resulting in low variance among the subgroup means or the ones with low average subgroup variance.

\begin{itemize}
    \item \textit{``Anti-clustering''}: On the one hand, it is clear that the solutions that choose to maximize $\frac{1}{|Q|} \sum_{q \in Q} \var(X_q)$ share the same spirit as anti-clustering. But such maximization of variance often results in scale differences among the subgroups. To see the potential scale problem, Lemma \ref{l:subsample_BV_tradeoff} shows that maximization of $\frac{1}{|Q|} \sum_{q \in Q} \var(X_q)$ necessarily leads to minimization of $\var(\mathbb{E}(X_Q))$. Therefore, in order to enforce low variance in the subgroup averages, the partition tends to group data points sharing the same scale together to achieve balance within the subgroups and have subgroups' averages close to the sample average.
    \item \textit{Barycenter matching}: On the other hand, if one chooses to minimize the average variance and enforce the scale similarity among the resulting subgroups, Lemma \ref{l:subsample_BV_tradeoff} shows that it is necessary to increase the variance of the subgroup averages. The following result shows that the multi-marginal matching in constructing the barycenter of $X_P$ provides a solution that coincides with the extremal solution of the trade-off that minimizes the average variance.
\end{itemize}

\begin{thm}[Barycenter of $X_P$ equals $\mathbb{E}(X_Q)$ with the largest variance in $\mathcal{Q}(P)$]\label{th:max_var_characterization}
    For all $Q \in \mathcal{Q}(P)$, we have
    \begin{equation}
        \var(\mathbb{E}(X_Q)) \leq \var(\bar{X}_P).
    \end{equation}
    Furthermore, let $T_p$ denote the optimal transport map between $\bar{X}_P$ and $X_p$ for all $p \in P$, then the equality holds as
    \begin{equation}
        \mathcal{W}_2(\mathbb{E}(X_{Q(\{T_p\}_p)}), \bar{X}_P) = 0.
    \end{equation}
\end{thm}

See proof in Appendix \ref{A:4}. It is then straightforward to combine the above two results to show that $X_{Q(\{T_p\}_p)}$ is the partition in $\mathcal{Q}(P)$ that minimizes the average variance:

\begin{coro}\label{coro: WHOMP_Matching}
Let $Q':= Q(\{T_p\}_p)$ as constructed in Theorem \ref{th:max_var_characterization}. For all $Q \in \mathcal{Q}(P)$, we have
\begin{equation}
    \frac{1}{|Q|} \sum_{q \in Q'} \var(X_q) \leq \frac{1}{|Q|} \sum_{q \in Q} \var(X_q)
\end{equation}
\end{coro}

We finish this section by showing that WHOMP matching, which is the extremal WHOMP solution minimizing expected in-subgroup variance (Algorithm \ref{algo:WHOMP_matching}), gives the desired solution in Example~\ref{exam:variance_homogeneity}.

\begin{rema}[Homogeneity partition gives desirable subgroups in Example \ref{exam:variance_homogeneity}]
Continuing with Example \ref{exam:variance_homogeneity}, the marginal subgroups can be obtained by performing 3-means on $\{x_i\}_{i \in [9]}$ to obtain $$P = \{p_1 = \{x_1,x_2,x_3\}, p_2 = \{x_4,x_5,x_6\}, p_3 = \{x_7,x_8,x_9\}\}.$$
Now, to find the extremal WHOMP solution minimizing expected in-subgroup variance, we first find the Wasserstein barycenter of $\{p_i\}_{i \in [3]}$. Then, for each point on the barycenter, we find the pre-images to form $$Q = \{q_1 = \{x_1,x_4,x_7\}, q_2 = \{x_2,x_5,x_8\}, q_1 = \{x_3,x_6,x_9\} \},$$ which is exactly the desirable partition. See also Figure~\ref{anti vs desire}.
\end{rema}

\section{Algorithm Design}\label{S:5}

In this section, we present two algorithms for WHOMP (Definition \ref{d:homogeneity_parition}) solutions based on our theoretical results in Section \ref{S:4} and one algorithm that we used to estimate the Balanced K-means (Definition \ref{d:balanced_kmeans}) solution:
\begin{itemize}
    \item \textit{WHOMP Random}: Applying randomness to balance the averages' variance (variance of subgroup expectations) and variances' average (expected in-subgroup variance).
    \item \textit{WHOMP Matching}: Applying Wasserstein matching to minimize the expected in-subgroup variances.
    \item \textit{Balanced K-means}: Applying optimal transport with Lloyd's algorithm to find K-means solution with uniform cardinality constraint on clusters.
\end{itemize}

\subsection{WHOMP Random:}

\begin{algorithm}
\SetAlgoLined
\caption{WHOMP Random}
\label{algo: WHOMP_random}
\begin{flushleft}
{\bf Input:} sample data set $\{X_i\}_{i = 1}^N$;\\

{\bf Step 1:} Balanced K-means clustering:\;
Obtain balanced K-means clustering (K-means with uniform cardinality constraint) on $\{X_i\}_{i = 1}^N$: $$P := \{p_k\}_{k = 1}^K.$$

{\bf Step 2:} Random selection without replacement\;
 \While{$j \in [\frac{N}{K}]$}{
  Draw $x'_k \in p_k$ without replacement, for each $k \in [K]$\;
  Form $q_i := \{x'_k\}_{k \in K}$\;
 }
{\bf Output:} $Q := \{q_i\}_{i = 1}^{\frac{N}{K}}.$
\end{flushleft}
\end{algorithm}

WHOMP Random is defined as Algorithm \ref{algo: WHOMP_random}. Here, balanced K-means clustering refers to standard K-means clustering with the additional constraint that all resulting clusters must have equal cardinality (i.e., each cluster contains the same number of points). In the implementation used for the numerical experiments in Section \ref{S:6}, we employ Algorithm \ref{algo: Balanced_Kmeans} to estimate the balanced K-means clustering solution. Our algorithm design is inspired by the size-constrained distance clustering \cite{bradley2000constrained}, and our implementation is inspired by the minimum flow approach to estimate size-constrained K-means.

\begin{algorithm}
\SetAlgoLined
\caption{Balanced K-means}
\label{algo: Balanced_Kmeans}
\begin{flushleft}
{\bf Input:} sample data set $\{X_i\}_{i = 1}^N$, request number of clusters $K$, max iteration, threshold;

{\bf Step 1:} Random initialization: 
Obtain $K$ centers: $\{ \bar{x}_i \}_{i = 1}^k$, iteration number: iter = 0, approximation difference: $\epsilon = \infty$ \;

{\bf Step 2:} Iterative Optimal Transport:\;
 \While{ $\text{iter} \leq \text{max iteration}$ and $\epsilon > \text{threshold}$ }{
  Find the optimal transport from $\frac{1}{N} \sum_{i = 1}^N \delta_{X_i}$ to $\frac{1}{N} \sum_{i = 1}^K c \delta_{\bar{x}_i}$, denoted by $T$\;
  Find the pre-images of the optimal transport map for each $\bar{x}_i$: $T^{-1}(\bar{x}_i)$ \;
  Compute the centroid of the preimages for each $\bar{x}_i$: $\bar{x}'_i := \frac{1}{c} \sum_{x \in T^{-1}(\bar{x}_i)} x$ \;
  Update $\epsilon = \mathcal{W}_2^2( \frac{1}{K} \sum_{i = 1}^K \delta_{\bar{x}_i}, \frac{1}{K} \sum_{i = 1}^K \delta_{\bar{x}'_i})$ \;
  Update $\text{iter} = \text{iter} + 1 $ \;
  Update $\bar{x}_i = \bar{x}'_i, \forall i \in [K]$
 }
{\bf Output:} $P = \{p_i\}_{i \in [K]}$ where $p_i := T^{-1}(\bar{x}_i), \forall i \in [K]$.
\end{flushleft}
\end{algorithm}

In each of the iteration steps in the Balanced K-means (Algorithm \ref{algo: Balanced_Kmeans}), the uniform cardinality constraint is enforced by the uniform weight assigned to each of the centers: If we consider mapping $K*c$ points each with weight $\frac{1}{K*c}$ to $K$ target points each with weight $\frac{1}{K}$ equivalent to mapping $K*c$ points each with weight $\frac{1}{K*c}$ to $c$ copies of the $K$ target points each with weight $\frac{1}{K*c}$, then Choquet's minimization theorem and Birkhoff's theorem together implies that the optimal transport plan is a permutation matrix. Therefore, the optimal transport map assigns $c$ points to each target point.

\subsection{WHOMP Matching}

We now introduce the WHOMP Matching algorithm (Algorithm \ref{algo:WHOMP_matching}) to obtain the extremal WHOMP solution that minimizes the expected within-subgroup variance based on Corollary~\ref{coro: WHOMP_Matching}.

\begin{algorithm}
\SetAlgoLined
\caption{WHOMP Matching}
\label{algo:WHOMP_matching}
\begin{flushleft}
{\bf Input:} sample data set $\{X_i\}_{i = 1}^N$: $P := \{p_k\}_{k = 1}^K$;

{\bf Step 1:} Balanced K-means clustering:\;
Obtain balanced K-means clustering (K-means with uniform cardinality constraint) on $\{X_i\}_{i = 1}^N$: $$P := \{p_k\}_{k = 1}^K.$$

{\bf Step 2:} Barycenter of K-means clusters:\;
Find the Wasserstein barycenter of the clusters in $P$ obtained in Step 1, denoted by $$\bar{X} := \{\bar{x}_i\}_{i = 1}^{\frac{N}{K}},$$ and the corresponding optimal transport map $T_k$ that maps $\bar{X}$ to $X_{p_k}$ for each $k \in [K]$.

{\bf Step 3:} Group the pre-images of barycenter\;
 \While{$i \in [\frac{N}{K}]$}{
  Form $q_i := \{T_k(\bar{x}_i)\}_{k \in [K]}$\;
 }
{\bf Output:} $Q := \{q_i\}_{i = 1}^{\frac{N}{K}}.$
\end{flushleft}

\end{algorithm}

\section{Numerical Experiments} \label{S:6}

The code for the WHOMP (random and matching) implementations, along with the numerical experiments, is available at \url{https://github.com/xushizhou/WHOMP}.

In this section, we compare the proposed subsampling/partition method, WHOMP, with two baselines: random partitioning and covariate-adaptive randomization (Pocock and Simon’s method), using the following datasets:

\begin{itemize}
\item Tabular data: synthetic data generated from a Gaussian mixture model
\item Tabular data: NPI dataset\footnote{Raw data from online personality tests: Narcissistic Personality Inventory. Available at the Open-Source Psychometrics Project website: \url{https://openpsychometrics.org/_rawdata/}}
\item Image data: MNIST \cite{lecun2010mnist}
\item Graph data: synthetic data generated by a stochastic block model
\end{itemize}

Before presenting our experimental results, it is important to note that while WHOMP works efficiently for data in moderate or low-dimensional Euclidean spaces, it can also be applied to various data formats by embedding the data into Euclidean space. In this section, we use eigenvectors of the graph Laplacian to embed graph data into Euclidean form. Additionally, we apply t-SNE to embed high-dimensional image data into a lower-dimensional Euclidean space.

\subsection{Tabular Data: Gaussian Mixture Model}

In this experiment, we test four partition methods: random partitioning, covariate-adaptive randomization (Pocock and Simon’s method), WHOMP random, and WHOMP matching, using synthetic data generated from a Gaussian mixture model. To compare the subgroup homogeneity produced by these methods, we perform the following downstream tasks using the subgroups generated by each partitioning method. For all tests, the sample size is fixed at 60 to prevent the law of large numbers. In cases where the law of large numbers applies, one would not expect significant differences between subgroups generated by different partition methods. Additionally, the number of subgroups is fixed at ${2, 4, 6}$ for all experiments. However, it should be noted that the sample size, number of subgroups, and Gaussian mixture model parameters are all arbitrarily chosen, and we encourage readers to explore WHOMP with different sample sizes and subgroup numbers on other datasets.

\subsubsection{Wasserstein-2 Distance Experiment}
The goal of this experiment is to validate the theoretical results by comparing the average Wasserstein-2 distance between the subgroups generated by each partition method and the original sample, across 100 repeated tests. A lower average distance and lower variance indicate a better partition method.

Specifically, for each repetition, we begin by randomly drawing 60 data points as the sample dataset, with each 20 points randomly sampled from $\mathcal{N}((0,10),3\Id)$, $\mathcal{N}((-10,-5),3\Id)$, and $\mathcal{N}((10,-5),3\Id)$. We then apply each partition method to this sample to generate the required number of subgroups. Finally, we compute the average Wasserstein-2 distance between the resulting subgroups and the original sample for each partition method.

Table~\ref{table:1} summarizes the mean and standard deviation of the average Wasserstein-2 distances (between the subgroups and the original sample) across the 100 repetitions.

\begin{table}[htbp]
    \centering
\textbf{Average (std) W-2 distance between the sample and subgroups}
\begin{tabularx}{\textwidth} { bsss }
 \hline
 \textbf{Partition method} & \textbf{2 subgroups} & \textbf{4 subgroups} & \textbf{6 subgroups} \\
 \hline
 \textbf{Random} & 3.481 (0.890)  & 5.600 (0.849) & 6.665 (0.692)\\
 \hline
 \textbf{Covariate-adaptive} & 3.589 (0.939) & 5.469 (0.886) & 6.566 (0.771) \\
 \hline
 \textbf{WHOMP random} & 1.642 (0.146) & 2.575 (0.200) & 4.029 (0.211)\\
 \hline
 \textbf{WHOMP matching} & 1.651 (0.145) & 2.634 (0.199) & 4.170 (0.225)\\
 \hline
\end{tabularx}
\caption{In the table above, we present the mean and standard deviation of the average Wasserstein-2 distances (between the resulting subgroups and the original sample) across the 100 repetitions. The results clearly show that the WHOMP solutions yield both lower average Wasserstein-2 distances and lower standard deviations compared to the random partitioning and covariate-adaptive randomization (Pocock and Simon’s method).}
\label{table:1}
\end{table}

Figure~\ref{f:GaussianW2plots} illustrates the exact distribution of the average Wasserstein-2 distance between the original sample and the resulting subgroups across the 100 repetitions.

\begin{figure}[H]
\centering
\includegraphics[width=0.49\textwidth]{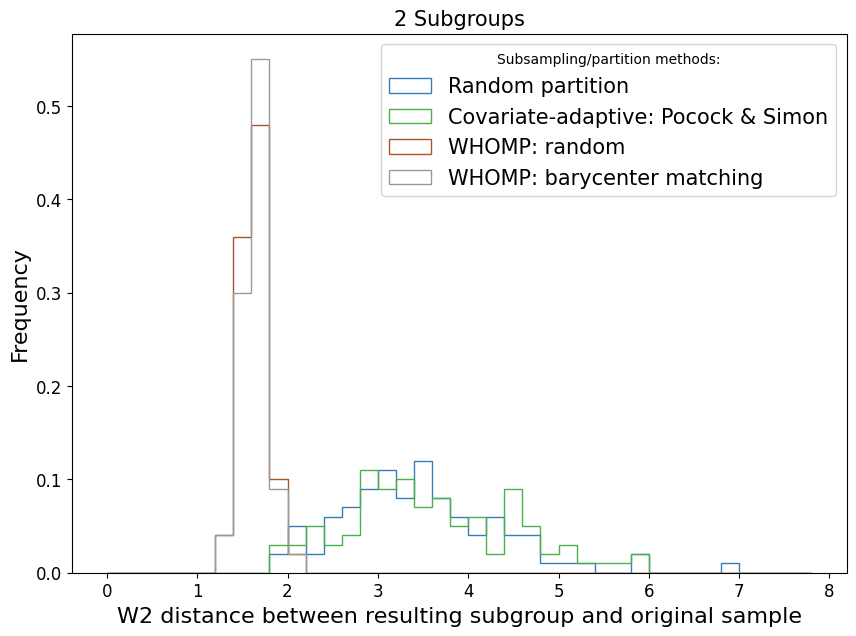}
\includegraphics[width=0.49\textwidth]{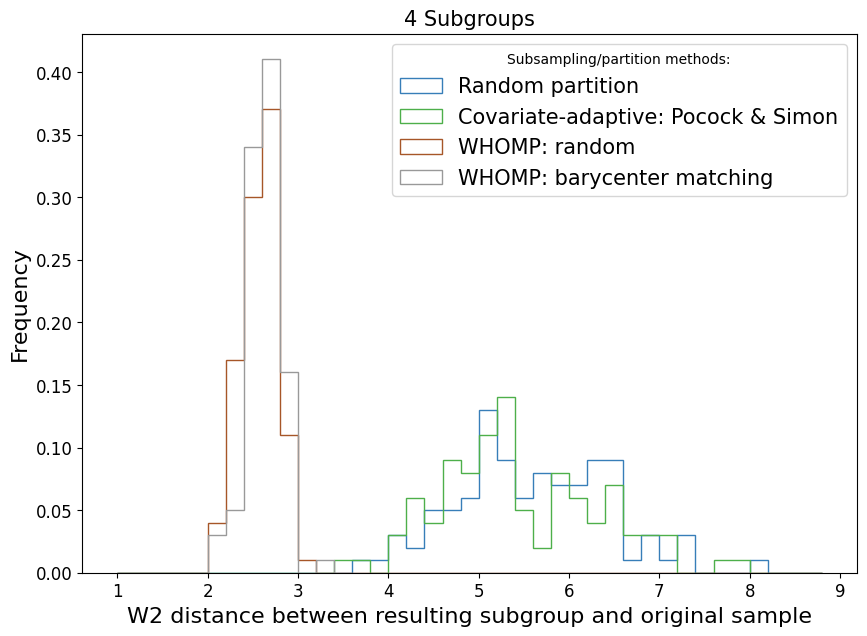}\\
\includegraphics[width=0.49\textwidth]{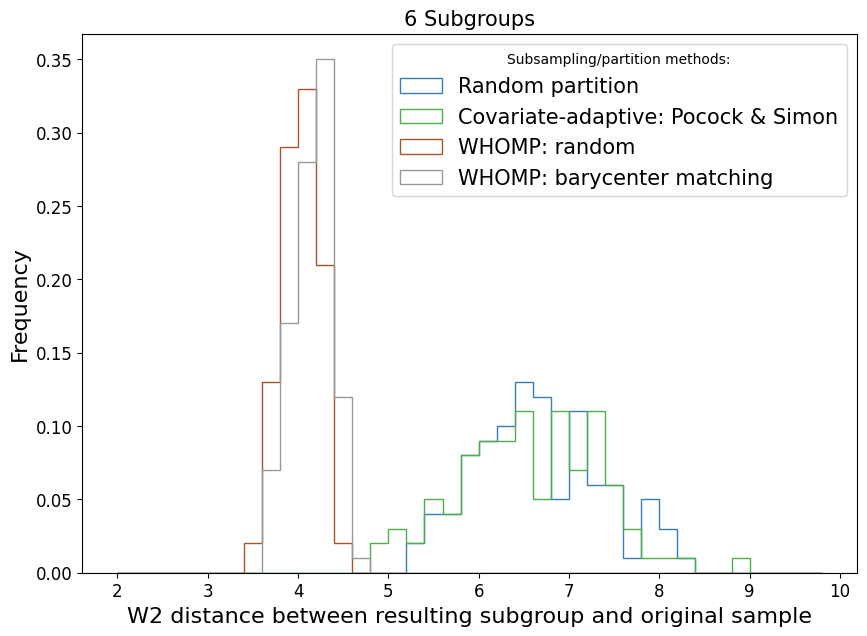}\\
\caption{It is evident from the frequency plot above that the worst-case Wasserstein distances resulting from WHOMP solutions are almost the best-case Wasserstein distance resulting from the random partition or Pocock $\&$ Simon's covariate-adaptive randomization.}
\label{f:GaussianW2plots}
\end{figure}

Furthermore,  Table~\ref{table:2} presents the variance of the first two moments of the subgroups resulting from the four partition methods. This illustrates the theoretical trade-off, Theorem \ref{th:max_var_characterization}, between the variances of the first two moments in the WHOMP solutions.

\begin{table}[htbp]
    \centering
\textbf{Variance of subgroups' 1st moments (2nd moments)}
\begin{tabularx}{\textwidth} { bsss }
 \hline
 \textbf{Partition method} & \textbf{2 subgroups} & \textbf{4 subgroups} &  \textbf{6 subgroups} \\
 \hline
 \textbf{Random} & 0.651 (33.376)  & 2.147 (80.661) & 2.990 (134.742)\\
 \hline
 \textbf{Covariate-adaptive} & 0.629 (40.929) & 1.787 (80.221) & 2.873 (129.885) \\
 \hline
 \textbf{WHOMP random} & 0.093 (23.422) & 0.130 (35.381) & 0.302 (57.638)\\
 \hline
 \textbf{WHOMP matching} & 0.199 (21.725) & 0.555 (27.972) & 1.358 (36.988)\\
 \hline
\end{tabularx}
\caption{In the table above, we present the variance of the first and second moments of the subgroups resulting from the different partition methods, averaged over repeated tests where each repetition draws the original sample randomly from Gaussian mixture models. The results reveal a trade-off between the variances of the first and second moments in the WHOMP solutions: WHOMP random exhibits lower variance in the first moment but higher variance in the second moment, while WHOMP matching shows higher variance in the first moment but lower variance in the second moment. Additionally, Pocock and Simon’s covariate-adaptive randomization achieves low variance in subgroup averages, aligning with its algorithmic objective. However, it results in a similarly high variance in the second moments as the purely randomized partition method, indicating significant distributional discrepancies.}
\label{table:2}
\end{table}

\subsubsection{Classification Experiment: Logistic Regression and SVM}
The goal is to assess the distributional homogeneity of the subgroups by using logistic regression or SVM trained on one randomly chosen subgroup to predict the true labels (from the Gaussian mixture model) on another randomly chosen subgroup. For each partition method, we repeat the prediction test 100 times. Higher average prediction accuracy and lower prediction accuracy variance indicate a better partition method.

For each partition method and repetition, we generate the sample set by randomly selecting 20 data points from $\mathcal{N}((0,10),4\Id)$, $\mathcal{N}((-10,-5),4\Id)$, and $\mathcal{N}((10,-5),4\Id)$. We then apply the partition methods to this sample set to form subgroups, randomly select (without replacement) two of the resulting subgroups as the training and test sets, train a logistic regression model on the training set, and record the test accuracy on the test set.

For the SVM experiment, the goals and test design are identical to those of the logistic regression test, except that logistic regression is replaced with SVM, and the Gaussian mixture model is replaced with $\mathcal{N}((0,10),4\Id)$, $\mathcal{N}((-10,-5),2\Id)$, and $\mathcal{N}((10,-5),2\Id)$.

Table~\ref{table:3}  summarizes the outcomes of the logistic regression and SVM experiments.

\begin{table}[htbp]
    \centering
\textbf{Logistic Regression Test Accuracy (Standard Deviation)}
\begin{tabularx}{\textwidth} { bsss }
 \hline
 \textbf{Partition method} & \textbf{2 subgroups} & \textbf{4 subgroups} &  \textbf{6 subgroups} \\
 \hline
 \textbf{Random} & 0.981 (0.023)  & 0.983 (0.032) & 0.944 (0.117)\\
 \hline
 \textbf{Covariate-adaptive} & 0.979 (0.017) & 0.971 (0.045) & 0.950 (0.073) \\
 \hline
 \textbf{WHOMP random} & 0.985 (0.017) & 0.982 (0.030) & 0.980 (0.040)\\
 \hline
 \textbf{WHOMP matching} & 0.983 (0.017) & 0.977 (0.032) & 0.984 (0.037)\\
 \hline
\end{tabularx}

\vspace{0.5cm}

\textbf{Support Vector Machine Test Accuracy (Standard Deviation)}
\begin{tabularx}{\textwidth} { bsss }
 \hline
 \textbf{Partition method} & \textbf{2 subgroups} & \textbf{4 subgroups} &  \textbf{6 subgroups} \\
 \hline
 \textbf{Random} & 1.000 (0.000)  & 0.993 (0.054) & 0.987 (0.054)\\
 \hline
 \textbf{Covariate-adaptive} & 1.000 (0.000) & 0.998 (0.020) & 0.982 (0.061) \\
 \hline
 \textbf{WHOMP random} & 1.000 (0.000) & 1.000 (0.000) & 0.999 (0.010)\\
 \hline
 \textbf{WHOMP matching} & 1.000 (0.000) & 1.000 (0.000) & 1.000 (0.000)\\
 \hline
\end{tabularx}
\caption{The table above clearly shows that the WHOMP solutions result in higher average test accuracy with lower standard deviation compared to the other methods. Additionally, the difference becomes more significant as the number of subgroups increases (or equivalently, as the subgroup sizes decrease). Therefore, the advantages of WHOMP are more significant when multiple controlled factors need to be tested.}
\label{table:3}
\end{table}

\subsubsection{Regression Experiment: Linear Regression}
The goal is to evaluate the distributional homogeneity of the subgroups by using a linear regression model trained on one randomly chosen subgroup to predict feature variables on another randomly chosen subgroup. A lower average mean squared error (MSE) and lower error variance indicate a better partition method. More specifically, the test design is the same as the classification experiment, with the distinction that one feature variable in the sample data is chosen as the dependent variable to predict, while the remaining variables serve as independent variables. The results of the experiment are summarized in Table~\ref{table:4}.

\begin{table}
\centering
\textbf{Linear Regression Test MSE error (Standard Deviation)}
\begin{tabularx}{\textwidth} { bsss }
 \hline
 \textbf{Partition method} & \textbf{2 subgroups} & \textbf{4 subgroups} &  \textbf{6 subgroups} \\
 \hline
 \textbf{Random} & 1.355 (0.112)  & 1.943 (0.298) & 2.548 (0.563)\\
 \hline
 \textbf{Covariate-adaptive} & 1.351 (0.105) & 1.988 (0.343) & 2.383 (0.370) \\
 \hline
 \textbf{WHOMP random} & 1.291 (0.034) & 1.875 (0.142) & 2.282 (0.213)\\
 \hline
 \textbf{WHOMP matching} & 1.309 (0.042) & 1.861 (0.171) & 2.320 (0.230)\\
 \hline
\end{tabularx}
\caption{In the table above, the test mean squared error (MSE) is obtained by training a linear regression model on one randomly chosen subgroup (resulting from the corresponding partition method) and then testing it on another randomly chosen subgroup. The test accuracy average and standard deviation are calculated from 100 repeated tests, with each repetition involving a random draw of the original sample from Gaussian mixture models. It is evident that the WHOMP solutions result in lower MSE and lower standard deviation compared to the random partition and Pocock and Simon’s covariate-adaptive randomization methods.}
\label{table:4}
\end{table}

\subsection{Tabular Data: NPI Data Set}

We test WHOMP on the NPI dataset, a real-world dataset used in \cite{papenberg2024k} to demonstrate distributional similarity among subgroups. Specifically, for each partition method, we first randomly select 60 data points from the NPI dataset as the sample. We then apply the partition method to generate subgroups and compute the Wasserstein-2 distance between the resulting subgroups and the sample. The sample size is kept small to highlight differences before the Law of Large Numbers affects the results in randomly subsampled data. Each test is repeated 500 times for each partition method. The results of the experiment are summarized in Table~\ref{table:5}.

\begin{table}[htbp]
    \centering
\textbf{Average (std) W-2 distance between the sample and subgroups}
\begin{tabularx}{\textwidth} { bsss }
 \hline
 \textbf{Partition method} & \textbf{2 subgroups} & \textbf{4 subgroups} & \textbf{6 subgroups} \\
 \hline
 \textbf{Random} & 2.222 (0.065)  & 2.808 (0.064) & 3.037 (0.063)\\
 \hline
 \textbf{Covariate-adaptive} & 2.227 (0.064) & 2.806 (0.064) & 3.036 (0.060) \\
 \hline
 \textbf{WHOMP random} & 2.148 (0.055) & 2.734 (0.062) & 2.965 (0.063)\\
 \hline
 \textbf{WHOMP matching} & 2.179 (0.055) & 2.751 (0.062) & 2.982 (0.062)\\
 \hline
\end{tabularx}
\caption{In the table above, we present the average and standard deviation of the Wasserstein-2 distance between the resulting subgroups and the original sample, computed over 500 repeated tests for each partition method. Each repetition involves drawing a sample randomly from the NPI dataset. It is evident that the WHOMP solutions achieve both a lower average Wasserstein-2 distance and a lower standard deviation compared to the random partition and Pocock and Simon’s covariate-adaptive randomization methods.}
\label{table:5}
\end{table}

Figure~\ref{f:GaussianW2plots500} shows the exact distribution of the Wasserstein-2 distance between the sample and the resulting subgroups across the 500 tests for the four partition methods.

\begin{figure}[H]
\centering
\includegraphics[width=0.49\textwidth]{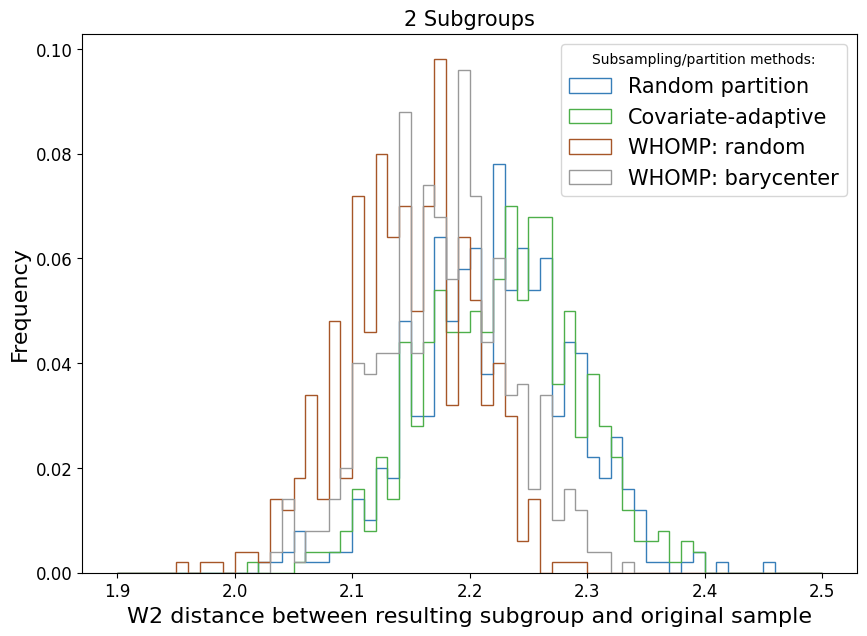}
\includegraphics[width=0.49\textwidth]{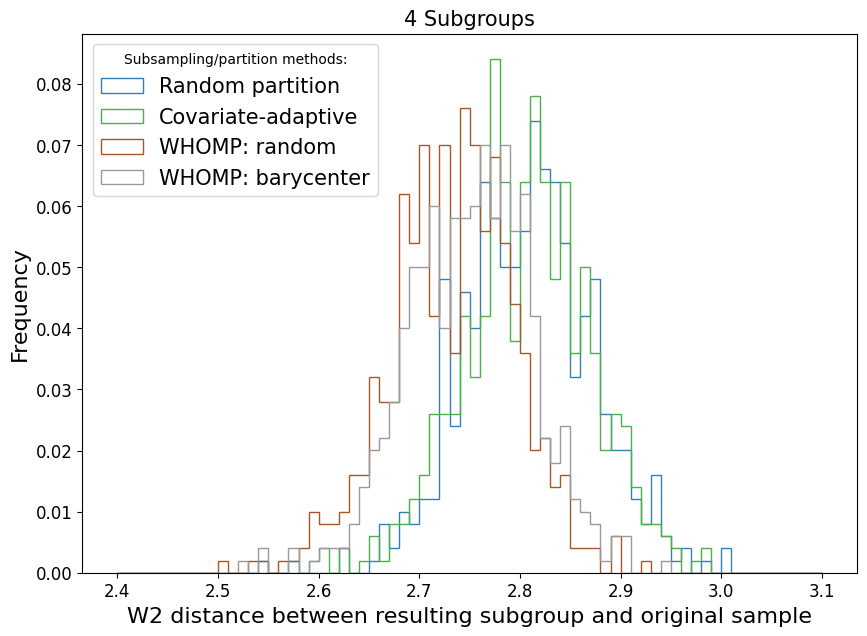}\\
\includegraphics[width=0.49\textwidth]{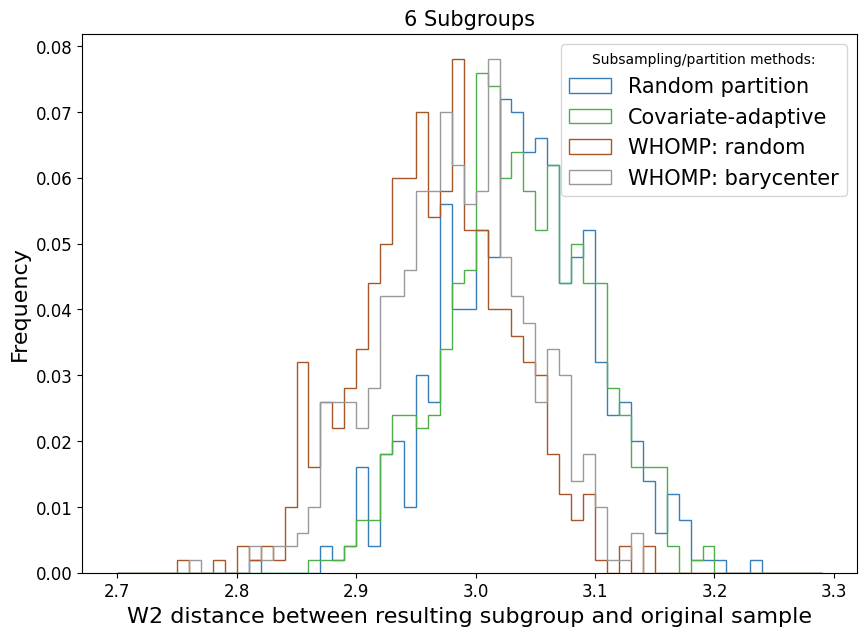}
\caption{It is evident from the frequency plot above that the worse-case Wasserstein distances resulting from WHOMP solutions are significantly better than the worst-case Wasserstein distance resulting from the random partition or Pocock $\&$ Simon's covariate-adaptive randomization. In the case of 2 subgroups, the 90th percentile worst-case distance in WHOMP solutions is equivalent to the 50th percentile worst-case distance in the other two methods.}
\label{f:GaussianW2plots500}
\end{figure}

\subsection{Image Data}

The goal of this experiment is to demonstrate that WHOMP effectively generates subgroups that are both diverse and homogeneous when applied to image data sets, after composing with embedding methods that embed high-dimensional image data to moderate or low-dimensional Euclidean space. Here, Euclidean closeness should imply closeness in the original image space.

In the experiment, we use t-SNE to embed the MNIST dataset into a 2-dimensional Euclidean space and then apply partition methods to generate partitions on the original image dataset. Due to memory constraints, we reduce the MNIST dataset to 10,000 images in this experiment.

Table~\ref{table:6} presents the mean and standard deviation (over 50 repeated tests) for the normalized entropy of the subgroup frequency vectors. For each subgroup, the frequency vector is defined such that each of the ten entries represents the frequency of images corresponding to a particular digit within the subgroup. Higher normalized entropy values indicate a more balanced (uniformly distributed in terms of digit representation) or more diversified partitioning of the subgroups.

\begin{table}[htbp]
    \centering
\textbf{Average (std) of the Subgroup Normalized Entropy in MNIST}

\textbf{Baseline Sample Normalized Entropy = 0.9780}
\begin{tabularx}{\textwidth} { bsss }
 \hline
 \textbf{Partition method} & \textbf{2 subgroups} & \textbf{4 subgroups} & \textbf{6 subgroups} \\
 \hline
 \textbf{Random} & 0.963 (0.015)  & 0.945 (0.024) & 0.940 (0.026)\\
 \hline
 \textbf{Covariate-adaptive} & 0.963 (0.017) & 0.949 (0.017) & 0.941 (0.026) \\
 \hline
 \textbf{WHOMP random} & 0.972 (0.007) & 0.960 (0.012) & 0.959 (0.016)\\
 \hline
 \textbf{WHOMP matching} & 0.973 (0.008) & 0.961 (0.014) & 0.958 (0.017)\\
 \hline
\end{tabularx}
\caption{The table above shows the average and standard deviation of the normalized entropy of the resulting subgroups, computed over 50 repeated tests for each partition method. It is evident that the WHOMP solutions yield both higher average normalized entropy and lower standard deviation compared to the random partition and Pocock and Simon’s covariate-adaptive randomization methods.}
\label{table:6}
\end{table}

\subsection{Graph Data}

The goal is to demonstrate that WHOMP can be effectively applied to graph data when combined with embedding methods, such as spectral embedding. For each partition method, we perform the following steps and repeat the test 100 times:
\begin{enumerate}
    \item Generate random graphs from the stochastic block model with three blocks: each has a respective block size of 10, 20, and 30 with a respective edge probability of $[0.6, 0.2, 0.2]$, $[0.2, 0.6, 0.2]$, and $[0.2, 0.2, 0.6]$.
    \vspace{-0.2cm}
    \item Apply spectral embedding to map the graph data into a 2-dimensional Euclidean space.
    \vspace{-0.2cm}
    \item Use the fixed partition method to divide the graph into subgraphs.
    \vspace{-0.2cm}
    \item Compute the Wasserstein-2 distance between the spectrum of the graph Laplacian and the spectrum of the subgraph Laplacian.
    \vspace{-0.2cm}
    \item Compute the average and standard deviation of the Wasserstein-2 distances over all resulting subgraphs.
\end{enumerate}

Tables~\ref{table:7}~summarizes the mean and standard deviation (over 100 repeated trials) of the single-trial 1st moments (for each trial we compute the average Wasserstein-2 distance between the graph Laplacian spectrum and the resulting subgraph Laplacian spectra.) The goal is to show the expected closeness between the graph and subgraphs in a random single trial and how deviated the closeness is over repeated trials.

Tables~\ref{table:8}~summarizes the mean and standard deviation (over 100 repeated trials) of the single-trial (square root of) 2nd moments (for each trial we compute the standard deviation of the Wasserstein-2 distance between the graph Laplacian spectrum and the resulting subgraph Laplacian spectra.) The goal is to show the expected stability or uniformity of closeness between the graph and subgraphs in a random single trial and how deviated the uniformity is over repeated trials.

\begin{table}[htbp]
    \centering
\textbf{Mean (std) of the expected Wasserstein-2 distances between Graph and Subgraph Laplacian Spectra over }
\begin{tabularx}{\textwidth} { bsss }
 \hline
 \textbf{Partition method} & \textbf{2 subgroups} & \textbf{4 subgroups} & \textbf{6 subgroups} \\
 \hline
 \textbf{Random} & 11.072 (0.307)  & 16.704 (0.331) & 18.525 (0.253)\\
 \hline
 \textbf{Covariate-adaptive} & 11.108 (0.353) & 16.689 (0.289) & 18.591 (0.293) \\
 \hline
 \textbf{WHOMP random} & 11.275 (0.305) & 17.042 (0.262) & 18.942 (0.221)\\
 \hline
 \textbf{WHOMP matching} & 11.286 (0.200) & 16.563 (0.206) & 19.057 (0.249)\\
 \hline
\end{tabularx}
\caption{While achieving nearly the same average expected Wasserstein distance between subgraph and graph Laplacian spectra, WHOMP matching results in lower standard deviation over the 100 trials with randomly sampled graph. That implies WHOMP matching has better stability in partitioning different graphs.}
\label{table:7}
\vspace{0.5cm}

\textbf{Mean (std) of the Wasserstein-2 distance standard deviation between Graph and Subgraph Laplacian Spectra}
\begin{tabularx}{\textwidth} { bsss }
 \hline
 \textbf{Partition method} & \textbf{2 subgroups} & \textbf{4 subgroups} & \textbf{6 subgroups} \\
 \hline
 \textbf{Random} & 0.600 (0.429)  & 0.572 (0.251) & 0.583 (0.168)\\
 \hline
 \textbf{Covariate-adaptive} & 0.506 (0.330) & 0.645 (0.253) & 0.623 (0.219) \\
 \hline
 \textbf{WHOMP random} & 0.398 (0.303) & 0.572 (0.241) & 0.524 (0.167)\\
 \hline
 \textbf{WHOMP matching} & 0.410 (0.213) & 0.455 (0.182) & 0.438 (0.136)\\
 \hline
\end{tabularx}
\caption{By achieving lower average (over 100 trials) standard deviation (over subgraphs in each trial) of Wasserstein distance between subgraph and graph Laplacian spectra, WHOMP (especially matching) has better stability in both resulting subgraphs for each trial and partitioning different random graphs over the 50 trials.}
\label{table:8}
\end{table}

\section*{Future Directions}

We conclude this paper by briefly sketching some compelling extensions of the WHOMP framework.

\begin{itemize}
    \item \textbf{Sequentially Incoming Data}: Develop algorithms or subgroup assignment mechanisms that optimize the WHOMP objective for sequentially incoming data. This is particularly relevant in scenarios such as clinical trials where participants are enrolled sequentially over time.
    \item \textbf{Cross-Validation}: In essence, the WHOMP objective function quantifies the distributional deviation of the resulting subgroups from the original sample. Thus, by maintaining this function within a specified range (rather than strictly minimizing it, as done in this work), the resulting subgroups can be effectively used for training/testing splits in cross-validation or holdout set generation. This range can be chosen to reflect the realistic distributional variation between the sample and the population distribution.
\end{itemize}

\section*{Acknowledgement}

T.S.\ wants to acknowledge illuminating discussions with Philipp Beineke on the topic of randomized clinical trials.
The authors
acknowledge support from NSF DMS-2208356,  NIH R01HL16351, and 
P41EB032840.


\bibliography{References_diverse_subgroups}


\appendix

\section{Appendix: Proofs of Results in Section \ref{S:2}} \label{A:2}

\subsection{Proof of Theorem \ref{th:no_type_error}}

\begin{proof}
    Assume $Y(X,c_0) =_d Y(X,c_1)$ for all $c_0, c_1 \in \mathcal{C}$, then we have $\mathcal{W}_2(Y(X,c_0), Y(X,c_1)) = 0$. Now, by the assumption on the WHOMP objective that $\sum_i \mathcal{W}_2^2(X, X_{q_i}) = 0$, we have $\mathcal{W}_2(X, X_{q_i}) = 0, \forall i \in \{0,1\}$, which further implies $$\mathcal{W}_2(Y(X_{q_0},c_0), Y(X,c_1)) = 0 = \mathcal{W}_2(Y(X,c_0), Y(X_{q_1},c_1)).$$ But it follows from the triangle inequality that $$\mathcal{W}_2(Y(X_{q_0},c_0), Y(X_{q_1},c_1)) = 0,$$ which is equivalent to $Y(X_{q_0},c_0) =_d Y(X_{q_1},c_1)$. That completes the proof for the first statement.
    
    Now, we prove the second statement by contraposition. Let $c_0, c_1 \in \mathcal{C}$ be arbitrary and assume $Y(X_{q_0},c_0) =_d Y(X_{q_1},c_1)$. Then we have $\mathcal{W}_2(Y(X_{q_0},c_0), Y(X_{q_1},c_1)) = 0$. But we also have $$\mathcal{W}_2(Y(X_{q_0},c_0), Y(X,c_1)) = 0 = \mathcal{W}_2(Y(X,c_0), Y(X_{q_1},c_1))$$ from the assumption on the WHOMP objective $\sum_i \mathcal{W}_2^2(X, X_{q_i}) = 0$. It then follows from the triangle inequality that $$\mathcal{W}_2(Y(X,c_0), Y(X,c_1)) = 0,$$ which further implies that $Y(X,c_0) =_d Y(X,c_1)$. Finally, since our choice of $c_0, c_1 \in \mathcal{C}$ is arbitrary, that completes the proof for the second statement.
\end{proof}

\subsection{Proof of Lemma \ref{l:unbias_estimator}}

\begin{proof}
    The statement is a direct corollary of the equivalence between WHOMP Random and the rerandomization Lemma $\ref{l: rerandomization_equivalence}$ and \cite[Theorem 2.1]{10.1214/12-AOS1008}.
\end{proof}

\subsection{Proof of Corollary \ref{co:lipschitz_error_bound}}

\begin{proof}
    Assume $\frac{1}{|Q|} \sum_{q \in Q} \mathcal{W}_2^2(X_q,X) \leq d$, it follows
    \begin{align*}
        &(\frac{1}{|Q|} \sum_{q \in Q} \mathcal{W}_2(X_q,X))^2 \leq \frac{1}{|Q|^2} \sum_{q \in Q} \mathcal{W}_2^2(X_q,X) \leq d\\
        \implies & \frac{1}{|Q|} \sum_{q \in Q} \mathcal{W}_2(X_q,X) \leq \sqrt{d}\\
        \implies & \frac{1}{|Q|} \sum_{q \in Q} \mathcal{W}_1(X_q,X) \leq \sqrt{d}\\
        \implies & \mathbbm{P}(\{\mathcal{W}_1(X_q,X) > \epsilon\}) \leq \frac{\sqrt{d}}{\epsilon}\\
        \implies & \mathbbm{P}(\{\sup_{||h||_{L} \leq 1} |\mathbbm{E}(h(X_q)) - \mathbbm{E}(h(X))| > \epsilon\}) \leq \frac{\sqrt{d}}{\epsilon}.\\
        \implies & \mathbbm{P}(\{\sup_{||h||_{L} \leq L} |\mathbbm{E}(\frac{1}{L}h(X_q)) - \mathbbm{E}(\frac{1}{L}h(X))| > \epsilon\}) \leq \frac{\sqrt{d}}{\epsilon}\\
        \implies & \mathbbm{P}(\{\sup_{||h||_{L} \leq L} |\mathbbm{E}(h(X_q)) - \mathbbm{E}(h(X))| > L\epsilon\}) \leq \frac{\sqrt{d}}{\epsilon}\\
        \implies & \mathbbm{P}(\{\sup_{||h||_{L} \leq L} |\mathbbm{E}(h(X_q)) - \mathbbm{E}(h(X))| > \epsilon\}) \leq \frac{L\sqrt{d}}{\epsilon}.\\
    \end{align*}
    Here, the first line follows from Jensen's inequality, the third follows from $\mathcal{W}_1 \leq \mathcal{W}_2$, the forth from Markov inequality, and the fifth from the Kantorovich-Rubinstein duality.
\end{proof}

\subsection{Proof of Theorem \ref{th:error_bound_repeat_trial}}

\begin{proof}
    \begin{align*}
        \mathbbm{E}(||\hat{\tau}(Y,\mathcal{Q}) - \tau(Y)||_{l^2}^2) & = \mathbbm{E}(||\hat{\tau}(Y,\mathcal{Q})||_{l^2}^2) -  ||\tau(Y)||_{l^2}^2\\
        & = \mathbbm{E}(||\frac{|\mathcal{Q}|}{n} \sum_{i = 1}^n (Y_i(q)\mathcal{Q}_i(q) - Y_i(q')\mathcal{Q}_i(q')) ||_{l^2}^2) - ||\tau(Y)||_{l^2}^2.
    \end{align*}
    Here, the first equation follows from the Lemma \ref{l:unbias_estimator}, $\tau(Y) = \mathbbm{E}(\hat{\tau}(Y,\mathcal{Q}))$. Now, let $T_{Q}$ be the bijective map from $q$ to $q'$ satisfying: $$\mathcal{W}_2^2(X_{\mathcal{Q}(q)},X_{\mathcal{Q}(q')}) = \frac{|Q|}{n} \sum_{i \in q} ||X_i - X_{T_Q(i)}||_{l^2}^2.$$
    Then it follows that
    \begin{align*}
        & \mathbbm{E}(||\frac{|\mathcal{Q}|}{n} \sum_{i = 1}^n (Y_i(q)\mathcal{Q}_i(q) - Y_i(q')\mathcal{Q}_i(q')) ||_{l^2}^2)\\
        = & \mathbbm{E}(||\frac{|\mathcal{Q}|}{n} \sum_{i = 1}^n \mathcal{Q}_i(q)(Y_i(q) - Y_{T_{\mathcal{Q}}(i)}(q')) ||_{l^2}^2)\\
        \leq & \mathbbm{E} (\frac{|\mathcal{Q}|}{n} \sum_{i = 1}^n ||\mathcal{Q}_i(q) (Y_i(q) - Y_{T_{\mathcal{Q}}(i)}(q'))||_{l^2}^2)\\
        = & \mathbbm{E} ( \frac{|\mathcal{Q}|}{n} \sum_{i = 1}^n Q_i(q)||Y(x_i,q) - Y(x_{T_{\mathcal{Q}}(i)},q')||_{l^2}^2)\\
        \leq & \mathbbm{E} ( \frac{|\mathcal{Q}|}{n} \sum_{i = 1}^n L^2 Q_i(q) ||x_i - x_{T_{\mathcal{Q}}(i)}||_{l^2}^2)\\
        = & L^2 \mathbbm{E}( \mathcal{W}_2^2(X_{\mathcal{Q}(q)}, X_{\mathcal{Q}(q')}) ).
    \end{align*}
    Here, the first equation follows from the definition of $T_{\mathcal{Q}}$, the second from Jensen's inequality, the third from $\mathcal{Q}_i(q)$ being an indicator function, the fourth from the assumption of the uniform Lipschitz property of $Y$, the fifth from the definition of $T_{\mathcal{Q}}$. Furthermore, we have
    \begin{align*}
        L^2 \mathbbm{E}( \mathcal{W}_2^2(X_{\mathcal{Q}(q)}, X_{\mathcal{Q}(q')}) )
        \leq & L^2 \mathbbm{E}\Big( \frac{|\mathcal{Q}|}{n} \sum_{p \in P} \big(\sum_{i = 1}^n X_iP_i(p)\mathcal{Q}_i(q) - \sum_{i = 1}^n X_iP_i(p)\mathcal{Q}_i(q')\big)\Big)\\
        = & L^2 \frac{|\mathcal{Q}|}{n} \sum_{p \in P} \mathbbm{E}\Big( \sum_{i = 1}^n X_iP_i(p)\mathcal{Q}_i(q) - \sum_{i = 1}^n X_iP_i(p)\mathcal{Q}_i(q')\Big)\\
        \text{claim} \rightarrow = & L^2 \frac{|\mathcal{Q}|}{n} \sum_{p \in P} (\frac{2|Q|}{|Q| - 1} \Var(X_p))\\
        = & L^2 \frac{2|Q|}{|Q| - 1} (\frac{|\mathcal{Q}|}{n} \sum_{p \in P} \Var(X_p))\\
        = & L^2 \frac{2|Q|}{|Q| - 1} [\Var(X) - \Var(\mathbbm{E}(X_P))].
    \end{align*}
    Here, the first line follows from the construction of $\mathcal{Q}(P)$, the third from the claim that we will prove below, and the final from the law of total variance. Now, for any $Q \in \mathcal{Q}(P)$, we have
    \begin{align*}
        L^2 \frac{2|Q|}{|Q| - 1}  [\Var(X) - \Var(\mathbbm{E}(X_P))]
        = & L^2 \frac{2|Q|}{|Q| - 1} [\Var(X) - \Var(\mathbbm{E}(\bar{X}_Q))]\\
        = & L^2 \frac{2|Q|}{|Q| - 1} \frac{1}{2} \sum_{q \in Q} \mathcal{W}_2^2(X_q ,X)\\
        = & L^2 \frac{|Q|}{|Q| - 1} \sum_{q \in Q} \mathcal{W}_2^2(X_q ,X),
    \end{align*}
    where the first line follows from Lemma \ref{l:random_barycenter}, the second from the proof of Lemma \ref{l:barycenter_variance_characterization}.
    It remains to prove the claim. Indeed,
    \begin{align*}
        & \mathbbm{E}( \sum_{i = 1}^n X_iP_i(p)\mathcal{Q}_i(q) - \sum_{i = 1}^n X_iP_i(p)\mathcal{Q}_i(q'))\\
        = & \frac{1}{|Q|} \sum_{x \in X_p} \frac{1}{|Q| - 1} \sum_{x' \in X_p \setminus \{x\} } ||x - x'||^2_{l^2}\\
        = & \frac{1}{|Q|(|Q| - 1)} \sum_{x \in X_p} \sum_{x' \in X_p} ||x - x'||^2_{l^2}\\
        = & \frac{1}{|Q| - 1} \sum_{x \in X_p} \frac{1}{|Q|} \sum_{x' \in X_p} ||x - x'||^2_{l^2}\\
        = & \frac{1}{|Q| - 1} \sum_{x \in X_p} (||x - \mathbbm{E}(X_p)||_{l^2}^2 + \var(X_p) )\\
        = & \frac{1}{|Q| - 1} \sum_{x \in X_p} ||x - \mathbbm{E}(X_p)||_{l^2}^2 + \frac{1}{|Q| - 1} \sum_{x \in X_p} \var(X_p) \\
        = & \frac{|Q|}{|Q| - 1} \frac{1}{|Q} \sum_{x \in X_p} ||x - \mathbbm{E}(X_p)||_{l^2}^2 + \frac{|Q|}{|Q| - 1} \var(X_p) \\
        = & \frac{2|Q|}{|Q| - 1} \var(X_p)
    \end{align*}
    Here, the first equation follows from the construction of $\mathcal{Q}(P)$ and the forth from the fact that $\sum_{x \in X} ||y - x||^2_{l^2} = ||y - \mathbbm{E}(X)||_{l^2}^2 + \var(X)$. This completes the proof.
\end{proof}

\section{Appendix: Proofs of Results in Section \ref{S:3}} \label{A:3}

\subsection{Proof of Lemma \ref{l:random_sample_deviation_bound}}

\begin{proof}
    Assume for contradiction that there exists a $ \{x_{s,i}\}_{i \in [K]} =: X_{sample} \in \mathbbm{R}^{d \times K}$ such that $$\mathcal{W}_2^2(X,X_{sample}) < \min_{\substack{P \in \mathbf{P}(N,K) \\ |p| \equiv c}} \frac{1}{K} \sum_{p \in P} \var(X_p).$$ By Choquet's Minimization Theorem and Birkhoff's Theorem \cite{villani2021topics}, there exist optimal transport maps $\{T_{i}\}_{i \in [K]}$ such that each $T_i$ maps $c$ points in $X$ to $x_{s,i}$ for each $i \in [K]$. Therefore, the pre-images $T_i^{-1}(x_{s,i})$ satisfy $\bigcup_{i \in [K]} T_i^{-1}(x_{s,i}) = X$, $T_i^{-1}(x_{s,i}) \cap T_j^{-1}(x_{s,j}), \forall i \neq j$ and $|T_i^{-1}(x_{s,i})| = c$ for all $i \in [K]$. Therefore, $X_{P'} := \{X_{p'}\}_{i \in [K]} := \{T_i^{-1}(x_{s,i})\}_{i \in [K]}$ defines a partition on $X$ that satisfies $|p'| = c, \forall p' \in P'$. But, it follows that
    \begin{equation*}
        \frac{1}{K} \sum_{p' \in P'} \var(X_{p'}) \leq \mathcal{W}_2^2(X,X_{\text{sample}}) < \min_{\substack{P \in \mathbf{P}(N,K) \\ |p| \equiv c}} \frac{1}{K} \sum_{p \in P} \var(X_p).
    \end{equation*}
    This contradicts the definition of the right hand side. That completes the proof.
\end{proof}

\subsection{Proof of Lemma \ref{l: rerandomization_equivalence}}

\begin{proof}
    Since WHOMP Random generates random partitions from $\mathcal{Q}(P)$, it is equivalent to the accept and reject rule $\mathbbm{1}_{\mathcal{Q}(P)}(Q)$. Therefore, it suffices to show that $\mathbbm{1}_{\mathcal{Q}(P)}(Q) = \Phi(X,Q)$, or equivalently $$Q \in \mathcal{Q}(P) \iff \Var(\bar{X}_Q) = \Var(\mathbbm{E}(X_P)).$$
    $(\Rightarrow)$ First, assume $Q \in \mathcal{Q}(P)$. It then follows from Lemma \ref{l:random_barycenter} that $\Var(\bar{X}_Q) = \Var(\mathbbm{E}(X_P)$.
    $(\Leftarrow)$ Now, assume for contradiction that $Q \notin \mathcal{Q}(P)$ and $\Var(\bar{X}_Q) = \Var(\mathbbm{E}(X_P))$. It follows from Theorem \ref{th:max_var_characterization} that $\Var(\mathbbm{E}(X_{P'})) = \Var(\bar{X}_Q)$, where $X_{P'}:= \{X_{p'}\}_{p' \in P'}$, $X_{p'}:= \{T_q(\bar{x}_Q)\}_{q \in Q}$, and $T_q$ is the optimal transport map that maps $\bar{X}_Q$ to $X_q$ for each $q \in Q$. It follows from $Q \notin \mathcal{Q}(P)$ that $P' \neq P$. But that implies $\Var(\mathbbm{E}(X_{P'})) = \Var(\mathbbm{E}(X_{P}))$, which contradicts the uniqueness of $P$. This completes the proof.
\end{proof}

\subsection{Proof of Lemma \ref{l:anticlustering_random_duality}}

\begin{proof}

Pick an arbitrary $P \in \mathbf{P}(N,K)$ that satisfies $|p| \equiv c$. For the left hand side, we have:

\begin{align}
\sum_{p \in P} \Var(X_p) & = \sum_{p \in P} \Big(\frac{1}{c}\sum_{i \in p} ||x_i - \frac{1}{c}\sum_{j \in p}x_j||_{l^2}^2\Big) \nonumber\\
& = \sum_{p \in P} \Big(\frac{1}{c}\sum_{i \in [N]}||x_i||_{l^2}^2 - \frac{1}{c^2} \sum_{i,j \in p} \langle x_i,x_j \rangle_{l^2}\Big) \nonumber\\
& = \frac{1}{c}\sum_{i \in [N]} ||x_i||_{l^2}^2 - \frac{1}{c^2} \sum_{p \in P} \sum_{i,j \in p} \langle x_i,x_j \rangle_{l^2} \nonumber\\
& = \frac{1}{c}\sum_{i \in [N]} ||x_i||_{l^2}^2 - \frac{1}{c^2} \sum_{i,j \in [N]} \langle x_i,x_j \rangle_{l^2} \mathds{1}_{\{P(i) = P(j)\}}. \label{eqn:1}
\end{align}

Now, for the right hand side, we have:

\begin{align*}
\mathbb{E}_{\mathcal{Q}}\Big[\sum_{q \in Q} \Var(X_q)\Big] & = \sum_{Q \in \mathcal{Q}(P)} \mathbb{P}_{\mathcal{Q}}(Q)\Big[\frac{1}{K}\sum_{i \in [N]} ||x_i||_{l^2}^2 - \frac{1}{K^2} \sum_{q \in Q} \sum_{i,j \in q} \langle x_i,x_j \rangle_{l^2}\Big]\\
& = \frac{1}{K}\sum_{i \in [N]} ||x_i||_{l^2}^2 - \frac{1}{K^2} \sum_{Q \in \mathcal{Q}(P)} \mathbb{P}_{\mathcal{Q}}(Q)\Big[\sum_{q \in Q} \sum_{i,j \in q} \langle x_i,x_j \rangle_{l^2}\Big]\\
& = \frac{1}{K}\sum_{i \in [N]} ||x_i||_{l^2}^2 - \frac{1}{K^2} \sum_{Q \in \mathcal{Q}(P)} \mathbb{P}_{\mathcal{Q}}(Q)\Big[\sum_{i,j \in [N]} \langle x_i,x_j \rangle_{l^2} \mathds{1}_{\{Q(i) = Q(j)\}}\Big]\\
& =  \frac{1}{K}\sum_{i \in [N]} ||x_i||_{l^2}^2 - \frac{1}{K^2} \sum_{i,j \in [N]} \langle x_i,x_j \rangle_{l^2} \Big[\sum_{Q \in \mathcal{Q}(P)} \mathbb{P}_{\mathcal{Q}}(Q) \mathds{1}_{\{Q(i) = Q(j)\}}\Big].
\end{align*}

But

\begin{align*}
\sum_{Q \in \mathcal{Q}(P)} \mathbb{P}_{\mathcal{Q}}(Q) \mathds{1}_{\{Q(i) = Q(j)\}} & = \mathbb{P}_{\mathcal{Q}}(\{\mathcal{Q}(i) = \mathcal{Q}(j)\} | \mathcal{Q} \in \mathcal{Q}(P))\\
& =  \begin{cases} 0 &\mbox{if } P(i) = P(j) \\
\frac{1}{c} & \mbox{if } P(i) \neq P(j) \end{cases}.
\end{align*}

Therefore,

\begin{align*}
& \mathbb{E}_{\mathcal{Q}}[\sum_{q \in Q} \Var(X_q)]\\
= & \frac{1}{K}\sum_{i \in [N]} ||x_i||_{l^2}^2 - \frac{1}{KN} \sum_{i,j \in [N]} \langle x_i,x_j \rangle_{l^2} \mathds{1}_{\{P(i) \neq P(j)\}}\\
= & \frac{1}{K}\sum_{i \in [N]} ||x_i||_{l^2}^2 - \frac{1}{KN} \sum_{i,j \in [N]} \langle x_i,x_j \rangle_{l^2} + \frac{1}{KN} \sum_{i,j \in [N]} \langle x_i,x_j \rangle_{l^2} \mathds{1}_{\{P(i) = P(j)\}}
\end{align*}

Finally, combine the left and right hand sides, we obtain:

\begin{equation} \label{eq:anticlustering_random_duality}
\min_{\substack{P \in \mathbf{P}(N,K) \\ |p| \equiv c}} \sum_{p \in P} \Var(X_p) \iff \max_{\substack{P \in \mathbf{P}(N,K) \\ |p| \equiv c}} \mathbb{E}_{\mathcal{Q} \sim uniform(\mathcal{Q}(P)) }[\sum_{q \in Q} \Var(X_q)]
\end{equation}

\end{proof}

\subsection{Proof of Proposition \ref{prop:clustering_random_duality}}

\begin{proof}
For the right hand side, we have
\begin{align*}
\mathbb{E}_{\mathcal{P}}[\sum_{p \in \mathcal{P}} \Var(X_p)] & = \sum_{P \in \mathcal{P}(Q)} \mathbb{P}_{\mathcal{P}}(Q)\Big[\frac{1}{c}\sum_{i \in [N]} ||x_i||_{l^2}^2 - \frac{1}{c^2} \sum_{p \in P} \sum_{i,j \in q} \langle x_i,x_j \rangle_{l^2}\Big]\\
& =  \frac{1}{c} \sum_{i \in [N]} ||x_i||_{l^2}^2 - \frac{1}{c^2} \sum_{i,j \in [N]} \langle x_i,x_j \rangle_{l^2} \Big[\sum_{P \in \mathcal{P}(Q)} \mathbb{P}_{\mathcal{P}}(P) \mathds{1}_{\{P(i) = P(j)\}}\Big]\\
& = \frac{1}{c}\sum_{i \in [N]} ||x_i||_{l^2}^2 - \frac{K}{N^2} \sum_{i,j \in [N]} \langle x_i,x_j \rangle_{l^2} \mathds{1}_{\{Q(i) \neq Q(j)\}}.
\end{align*}
Here, the last equality follows from
\begin{align*}
\sum_{P \in \mathcal{P}(Q)} \mathbb{P}_{\mathcal{P}}(P) \mathds{1}_{\{P(i) = P(j)\}} & = \mathbb{P}_{\mathcal{P}}(\{\mathcal{P}(i) = \mathcal{P}(j)\}|\mathcal{P} \in \mathcal{P}(Q))\\
& = \begin{cases} 0 &\mbox{if } Q(i) = Q(j), \\
\frac{1}{K} & \mbox{if } Q(i) \neq Q(j). \end{cases}
\end{align*}
Now, for the left hand side, we have
\begin{align*}
\sum_{q\in Q} \Var(X_q) & = \sum_{q \in Q} (\frac{1}{K}\sum_{i \in q} ||x_i - \frac{1}{K}\sum_{j \in q}x_j||_{l^2}^2) \nonumber\\
& = \frac{1}{K}\sum_{i \in [N]} ||x_i||_{l^2}^2 - \frac{1}{K^2} \sum_{i,j \in [N]} \langle x_i,x_j \rangle_{l^2} \mathds{1}_{\{Q(i) = Q(j)\}}\\
& = \frac{1}{K}\sum_{i \in [N]} ||x_i||_{l^2}^2 - \frac{1}{K^2} \sum_{i,j \in [N]} \langle x_i,x_j \rangle_{l^2} + \frac{1}{K^2} \sum_{i,j \in [N]} \langle x_i,x_j \rangle_{l^2} \mathds{1}_{\{Q(i) \neq Q(j)\}}.
\end{align*}
Therefore, it follows from the left and right hand sides that
\begin{equation}
    \max_{\substack{Q \in \mathbf{P}(N,c) \\ |q| \equiv K }} \sum_{q \in Q} \Var(X_q)  \iff \max_{\substack{Q \in \mathbf{P}(N,c) \\ |q| \equiv K }} \mathbb{E}_{\mathcal{P} \sim uniform(\mathcal{P}(Q)) }[\sum_{p \in \mathcal{P}} \Var(X_p)]
\end{equation}

\end{proof}

\section{Appendix: Proofs of Results in Section \ref{S:4}} \label{A:4}

\subsection{Proof of Lemma \ref{l:barycenter_variance_characterization}}

\begin{proof}
The minimum and maximum below are all over the set $\{Q \in \mathbf{P}(N,c): |q| \equiv K\}$.
    \begin{align*}
        \min_{Q \in \mathbf{P}(N,c)} \sum_{q \in Q} \mathcal{W}_2^2(X_q,X)
        \equiv & \min_{Q \in \mathbf{P}(N,c)} \sum_{q \in Q} \sum_{q' \neq q} \mathcal{W}_2^2(X_q,X_{q'})\\
        \equiv & \min_{Q \in \mathbf{P}(N,c)} \sum_{q \in Q} \sum_{q' \in Q} \mathcal{W}_2^2(X_q,X_{q'})\\
        \equiv & \min_{Q \in \mathbf{P}(N,c)} \sum_{q \in Q} \mathcal{W}_2^2(X_q, \Bar{X}_Q)\\
        \equiv & \min_{Q \in \mathbf{P}(N,c)} [\var(X) - \var(\Bar{X}_Q) ]\\
        \equiv & \max_{Q \in \mathbf{P}(N,c)} [\var(\Bar{X}_Q) ]\\
    \end{align*}
Here, the first line follows from the optimality of optimal transport on subsets: $\mathcal{W}_2^2(X_q,X) = \sum_{q' \in Q} \mathcal{W}_2^2(X_q,X_{q'})$, the third line follows from the fact that $$2 \sum_{q \in Q} \mathcal{W}_2^2(X_q,\Bar{X}_Q) = \sum_{q \in Q} \sum_{q' \in Q} \mathcal{W}_2^2(X_q,X_{q'}),$$ the forth line follows from the variance reduction formulation: $\sum_{q \in Q} \mathcal{W}_2^2(X_q, \Bar{X}_Q) = \var(X) - \var(\Bar{X}_Q)$, and the last line follows from the fact that $\var(X)$ is constant.
\end{proof}

\subsection{Proof of Lemma \ref{l:random_barycenter}}

\begin{proof}
    Let $Q \in \mathcal{Q}(P)$ be arbitrary. For each $q \in Q$, let $T_q$ denote the optimal transport map that pushes $\mathcal{L}(\mathbbm{E}(X_P))$ to $\mathcal{L}(X_q)$. We claim that $T_q(\mathbbm{E}(X_p)) = X_{p \cap q}$ for all $p \in P$. It follows that
    \begin{equation*}
        \frac{1}{c} \sum_{q \in Q} T_q(\mathbbm{E}(X_p)) = \frac{1}{c} \sum_{q \in Q} X_{p \cap q} = \mathbbm{E}(X_p), \forall \mathbbm{E}(X_p) \in \mathbbm{E}(X_P).
    \end{equation*}
    It follows from the fixed point characterization of Wasserstein barycenter \cite{alvarez2016fixed} that $$\mathcal{W}_2(\mathbbm{E}(X_P), \bar{X}_Q) = 0.$$

    It remains to prove the claim. Indeed, assume for contradiction that the claim is not true, then there exists $q \in Q$ such that
    \begin{align*}
        \mathcal{W}_2^2(\mathcal{L}(\mathbbm{E}(X_P)), \mathcal{L}(X_q)) & = \frac{1}{K} \sum_{p \in P} ||\mathbbm{E}(X_p) - T_q(\mathbbm{E}(X_p))||^2\\
        & < \frac{1}{K} \sum_{p \in P} ||\mathbbm{E}(X_p) - X_{p \cap q}||^2
    \end{align*}
    But then define a new partition $P' \in \mathbf{P}(N,K)$ by $X_{p'} := \bigcup_{q \in Q} T_q(\mathbbm{E}(X_p))$ for each $p' \in P'$. It follows that
    \begin{align*}
        \frac{1}{K} \sum_{p' \in P'} \var(X_{p'}) & = \frac{1}{K} \sum_{p' \in P'} \frac{1}{c} \sum_{q \in Q} ||\mathbbm{E}(X_{p'}) - T_q(\mathbbm{E}(X_p))||^2\\
        & \leq \frac{1}{K} \sum_{p' \in P'} \frac{1}{c} \sum_{q \in Q} ||\mathbbm{E}(X_{p}) - T_q(\mathbbm{E}(X_p))||^2\\
        & < \frac{1}{K} \sum_{p' \in P'} \frac{1}{c} \sum_{q \in Q} ||\mathbbm{E}(X_{p}) - X_{p \cap q}||^2\\
        & = \frac{1}{K} \sum_{p \in P} \var(X_{p}).
    \end{align*}
    Here, the first line follows from the definition of $P'$, the second from the fact that Euclidean average is the Fréchet mean, and the third from the assumption. But that contradicts the optimality of $P$. Therefore, we have proved the claim by contradiction. Finally, since our choice of $Q \in \mathcal{Q}(P)$ is arbitrary, we are done.

\end{proof}

\subsection{Proof of Theorem \ref{th:max_var_characterization}}

\begin{proof}
Let $T'_p$ be the bijective map from $\mathbbm{E}(X_q)$ to $X_p \cap X_q$. For each $p \in P$, we have
\begin{align*}
    \frac{1}{|Q|}\sum_{q \in Q} \var(X_q) & = \frac{1}{c} \sum_{q \in Q} (\frac{1}{K}\sum_{x \in X_q} ||x - \mathbbm{E}(X_q)||^2)\\
    & = \frac{1}{c} \sum_{q \in Q} (\frac{1}{K} \sum_{p \in P} ||T'_p(\mathbbm{E}(X_q)) - \mathbbm{E}(X_q)||^2)\\
    & = \frac{1}{K} \sum_{p \in P} ( \frac{1}{c} \sum_{q \in Q} ||T'_p(\mathbbm{E}(X_q)) - \mathbbm{E}(X_q)||^2)\\
    & = \frac{1}{K} \sum_{p \in P} ||X_p - \mathbbm{E}(X_Q)||_2^2\\
    & \geq \frac{1}{K} \sum_{p \in P} \mathcal{W}_2^2(\mathcal{L}(X_p), \mathcal{L}(\bar{X}_P)).
\end{align*}
Here, the second line follows from the definition of $T'_p$, the penultimate line from the fact that $T'_{p_{\sharp}}\mathcal{L}(\mathbbm{E}(X_Q)) = \mathcal{L}(X_p)$ for all $p \in P$, and the last line follows from the definition of the Wasserstein-2 barycenter. Now, it follows from Lemma \ref{l:subsample_BV_tradeoff} that
    \begin{align*}
        \var(\mathbbm{E}(X_Q)) & = \var(X) - \frac{1}{|Q|} \sum_{q \in Q} \var(X_q)\\
        & \leq \var(X) - \frac{1}{K} \sum_{p \in P} \mathcal{W}_2^2(\mathcal{L}(X_p), \mathcal{L}(\bar{X}_P))\\
        & = \var(\bar{X}_P),
    \end{align*}
where the last line follows from the variance reduction of the Wasserstein-2 barycenter \cite[{Lemma 5.6}]{JMLR:v24:22-0005}.

Finally, when $T_p$ are the optimal transport maps from $\mathcal{L}(\bar{X}_P)$ to $\mathcal{L}(X_p)$, we have
\begin{equation}
    \mathbbm{E}(X_{q(\bar{x})}) = \frac{1}{K} \sum_{x \in X_{q(\bar{x})}} x = \frac{1}{K} \sum_{p \in P} T_p(\bar{x}) = \bar{x}.
\end{equation}
Since this is true for all $\bar{x} \in \Bar{X}_P$, we have $\mathcal{W}_2^2(\mathbbm{E}(X_{Q(\{T_p\}_p)}), \bar{X}_P) = 0$ which implies $\var(\mathbbm{E}(X_{Q(\{T_p\}_p)})) = \var(\bar{X}_P)$. That completes the proof.
\end{proof}

\end{document}